\documentclass[a4paper,12pt]{article}

\usepackage{float}
\usepackage{amsmath} 
\usepackage{amssymb}
\usepackage{amsfonts}
\usepackage{amsthm}
\usepackage{epstopdf}
\usepackage{color,fancybox,graphicx}
\usepackage{url}
\usepackage{psfrag}
\usepackage{color}
\usepackage{pifont}
\usepackage{mathrsfs}
\usepackage{dsfont}
\usepackage{algorithm, algorithmic}
\usepackage{enumerate}
\usepackage{xfrac}
\usepackage{psfrag}
\usepackage{pstricks, pst-tree}

\usepackage{tikz}
\usetikzlibrary{positioning}
\usetikzlibrary{arrows}
\tikzset{
  treenode/.style = {align=center, inner sep=0pt, text centered,
    font=\sffamily},
  arn_n/.style = {treenode, circle, black, draw=black,
    fill=white, text width=1.5em, thin},
  arn_r/.style = {treenode, circle, red, draw=red, 
    text width=1.5em, very thick},
  arn_b/.style = {treenode, circle,  draw=green, 
    text width=1.5em, very thick},
  arn_x/.style = {treenode, rectangle, draw=black, thick,
    minimum width=1.5em, minimum height=1.5em},
  arn_Leaf/.style = {treenode, fill = red, minimum size = 12pt, inner sep = 2pt, rectangle, draw=black, very thick, 
    minimum width=1.5em, minimum height=1.5em},
  emph/.style={edge from parent/.style={blue,thick,draw}},
    norm/.style={edge from parent/.style={black,thin,draw}}
}

\numberwithin{equation}{section}

\newtheorem{rem}{Remark}[section]
\newtheorem{defi}{Definition}[section]
\newtheorem{lem}{Lemma}[section]

\newtheorem{pro}{Proposition}[section]
\newtheorem{theo}{Theorem}[section]

\newcommand{\bX}{\mathbf{X}}
\newcommand{\bx}{\mathbf{x}}
\newcommand{\bz}{\mathbf{z}}

\DeclareMathOperator*{\argmax}{arg\,max}

\newcommand{\Mtry}{\mathscr{M}_{\textrm{try}}}

\newcommand{\V}{\mathds{V}}
\renewcommand{\P}{\mathds{P}}

\usepackage[textsize=small]{todonotes}

\definecolor{gris25}{gray}{0.90}

\usepackage{natbib}
\bibpunct{(}{)}{;}{a}{,}{,}

\usepackage{hyperref}
\hypersetup{
colorlinks = true,
urlcolor = blue, 
linkcolor = blue,
citecolor = blue,
}
\usepackage{enumerate}

\begin{document}

\begin{center}
{\Large 
\textbf{\textsf{Neural Random Forests}}}
\medskip
\medskip
\end{center}

{\bf G\'erard Biau}\\
{\it Sorbonne Universit\'e, CNRS, LPSM, Paris, France}\\
\href{mailto:gerard.biau@upmc.fr}{gerard.biau@upmc.fr}
\bigskip

{\bf Erwan Scornet}\\
{\it Centre de Math\'ematiques Appliqu\'ees, Ecole Polytechnique, CNRS,\\ 
Palaiseau, France}\\
\href{mailto:erwan.scornet@polytechnique.edu}{erwan.scornet@polytechnique.edu}
\bigskip

{\bf Johannes Welbl}\\
{\it University College London, London, England}\\
\href{mailto: J.Welbl@cs.ucl.ac.uk}{J.Welbl@cs.ucl.ac.uk}

\begin{abstract}
\noindent {\rm 
Given an ensemble of randomized regression trees, it is possible to restructure them as a collection of multilayered neural networks with particular connection weights. Following this principle, we reformulate the random forest method of Breiman (2001) into a neural network setting, and in turn propose two new hybrid procedures that we call neural random forests. Both predictors exploit prior knowledge of regression trees for their architecture, have less parameters to tune than standard networks, and less restrictions on the geometry of the decision boundaries than trees. Consistency results are proved, and substantial numerical evidence is provided on both synthetic and real data sets to assess the excellent performance of our methods in a large variety of prediction problems.
\medskip
 
\medskip

}

\end{abstract}

\section{Introduction}
\setcounter{MaxMatrixCols}{20}
Decision tree learning is a popular data-modeling technique that has been around for over fifty years in the fields of statistics, artificial intelligence, and machine learning. The approach and its innumerable variants have been successfully involved in many challenges requiring classification and regression tasks, and it is no exaggeration to say that many modern predictive algorithms rely directly or indirectly on tree principles. What has greatly contributed to this success is the simplicity and transparency of trees, together with their ability to explain complex data sets. The monographs by \citet{BrFrOlSt84}, \citet{DeGyLu96}, \citet{RoMa08}, and \citet{HaTiFr09} will provide the reader with introductions to the general subject area, both from a practical and theoretical perspective.

The history of trees goes on today with random forests \citep{Br01}, which are on the list of the most successful machine learning algorithms currently available to handle large-scale and high-dimensional data sets. This method works according to the simple but effective {\it bagging} principle: sample fractions of the data, grow a predictor (a decision tree in the case of forests) on each small piece, and then paste the results together. Although the theory is still incomplete, random forests have been shown to give state-of-the-art performance on a
number of practical problems and in different contexts \citep[e.g.,][]{Me06, IsKoChMi11}. They work fast, generally exhibit a substantial improvement over single tree learners, and yield generalization error rates that often rank among the best \citep[see, e.g.,][]{FeCeBaAm14}. The surveys by \citet{BoJaKrKo12} and \citet{BiSc15}, which include practical guidelines and updated references, are a good starting point for understanding the method.

It is sometimes alluded to that forests have the flavor of deep network architectures \citep[e.g.,][]{Be09}, insofar as ensemble of trees allow to discriminate between a very large number of regions. The richness of forest partitioning results from the fact that the number of intersections of the leaf regions can be exponential in the number of trees. That being said, the connection between random forests and neural networks is largely unexamined. On the one hand, the many parameters of neural networks make them a versatile and expressively rich tool for complex data modeling. However, their expressive power comes with the downside of increased overfitting risk, especially on small data sets. Conversely, random forests have fewer parameters to tune, but the greedy feature space separation by orthogonal hyperplanes results in typical stair or box-like decision surfaces, which may be advantageous for some data but suboptimal for other, particularly for colinear data with correlated features \citep{GuHoMaVa11}. In this context, the empirical studies by \citet{We14} and \citet{RiKaYaMyRo15} have highlighted the advantage of casting random forests into a neural network framework, with the intention to exploit the benefits of both approaches to overcome their respective shortcomings. 

In view of the above, the objective of this article is to reformulate the random forest method into a neural network setting, and in turn propose two new hybrid procedures that we call {\it neural random forests}. In a nutshell, given an ensemble of random trees, it is possible to restructure them as a collection of (random) multilayered neural networks, which have sparse connections and less restrictions on the geometry of the decision boundaries. 
Their activation functions are soft nonlinear and differentiable, thus trainable with a gradient-based optimization algorithm and expected to exhibit better generalization performance.
The idea of constructing a decision tree and using this tree to obtain a neural network is by no means new---see for example \citet{Se90,Se91}, \citet{Br91}, and \citet[][Chapter 30]{DeGyLu96}. 
Similar work in the intersection between decision trees and neural networks has been undertaken by \cite{Ko15}, who learn differentiable split functions to guide inputs through a tree. The \emph{conditional networks} from \cite{Io16} also use trainable routing functions to perform conditional transformations on the inputs---they are thus capable of transferring computational efficiency benefits of decision trees into the domain of convolutional networks. 
However, to our best knowledge, no theoretical study has yet been reported regarding the connection between random forests and networks.

This paper makes several important contributions. First, in Section 2, we show that any regression tree can be seen as a particular neural network. In Section 3, we exploit this equivalence to combine neural networks in a random forest style and define the two neural random forest predictors. Both predictors exploit prior knowledge of regression trees for their initialization. This provides a major advantage in terms of interpretability (compared to more ``black-box-like'' neural network models) and effectively offers the networks a warm start at the prediction level of traditional forests. Section 4 is devoted to the derivation of theoretical consistency guarantees of the two methods. We illustrate in Section 5 the excellent performance of our approach both on simulated and real data sets, and show that it outperforms random forests in most situations. For clarity, the most technical proofs are gathered in Section 6.
\section{Trees, forests, and networks}
The general framework is nonparametric regression estimation, in which an input random vector $\bX \in [0,1]^d$ is observed and the goal is to predict the square integrable random response $Y\in \mathds R$ by estimating the regression function $r(\bx)=\mathds E[Y|\bX=\bx]$. With this aim in mind, we assume we are given a training sample $\mathscr D_n=((\bX_1,Y_1), \hdots, (\bX_n,Y_n))$, $n \ge 2$, of independent random variables distributed the same as the independent prototype pair $(\bX, Y)$. The data set $\mathscr D_n$ is used to construct an estimate $r(\cdot\,;  \mathscr{D}_n): [0,1]^d \to \mathds R$ of the function $r$. We abbreviate $r(\bx\,;\mathscr{D}_n)$ to $r_{n}(\bx)$ and say that the regression function estimate $r_n$ is (mean squared error) consistent if $\mathds E |r_n(\bX)-r(\bX)|^2 \to 0$ as $n \to \infty$ (the expectation is evaluated over $\bX$ and the sample $\mathscr D_n$).
\subsection{From one tree to a neural network}
A regression tree is a regression function estimate that uses a hierarchical segmentation of the input space, where each tree node corresponds to one of the segmentation subsets in $[0,1]^d$. In the following, we consider ordinary binary regression trees. In this model, a node has exactly either zero children (in which case it is called a terminal node or leaf) or two children. If a node $u$ represents the set $A \subseteq [0,1]^d$ and its children $u_L$, $u_R$ (L and R for \textsl{Left} and \textsl{Right}) represent $A_L  \subseteq [0,1]^d$ and $A_R  \subseteq [0,1]^d$, then we require that $A = A_L \cup A_R$ and $A_L\cap A_R= \emptyset$. The root represents the entire space $[0,1]^d$ and the leaves, taken together, form a partition of $[0,1]^d$. In an ordinary tree, we pass from $A$ to $A_L$ and $A_R$ by answering a question on $\bx=(x^{(1)}, \hdots, x^{(d)})$ of the form: ``Is $x^{(j)}\geq \alpha$'', for some dimension $j\in\{1, \hdots, d\}$ and some $\alpha \in [0,1]$. Thus, the feature space $[0,1]^d$ is partitioned into hyperrectangles  whose sides are parallel to the coordinate axes. 
During prediction, the input is first passed into the tree root node. It is then iteratively transmitted to the child node that belongs to the subspace in which the input is located; this is repeated until a leaf node is reached.
If a leaf represents region $A$, then the natural regression function estimate takes the simple form $t_n(\bx)=(\sum_{i=1}^n Y_i\mathds 1_{\bX_i \in A})/N_n(A)$, $\bx \in A$, where $N_n(A)$ is the number of observations in cell $A$ (where, by convention, $0/0=0$). In other words, the prediction for a query point $\bx$ in leaf node $A$ is the average of the $Y_i$ of all training instances that fall into this region $A$. An example in dimension $d=2$ is depicted in Figure \ref{figure1}.
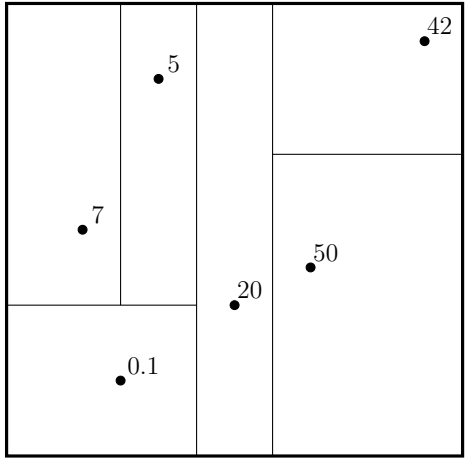
\begin{figure}[h!!]
    \centering
    \begin{tikzpicture}
 \tikzstyle{points}=[circle,  draw=black, fill ,scale=0.3]
 \tikzstyle{value}=[circle,  scale=0.8]
    
\draw[very thick] (0,0) rectangle (6,6);
\draw (2.5,0) -- (2.5, 6);
\draw (3.5,0) -- (3.5, 6);
\draw (0,2) -- (2.5,2);
\draw (1.5,2) -- (1.5,6);
\draw (3.5,4) -- (6, 4);

\node[points]  (I-1) at (1.5,1) {};
\node[value]  (I-2) at (1.8,1.2) {$0.1$};

\node[points]  (I-1) at (3,2) {};
\node[value]  (I-2) at (3.2,2.2) {$20$};

\node[points]  (I-1) at (2,5) {};
\node[value]  (I-2) at (2.2,5.2) {$5$};

\node[points]  (I-1) at (1,3) {};
\node[value]  (I-2) at (1.2,3.2) {$7$};

\node[points]  (I-1) at (4,2.5) {};
\node[value]  (I-2) at (4.2,2.7) {$50$};

\node[points]  (I-1) at (5.5,5.5) {};
\node[value]  (I-2) at (5.7,5.7) {$42$};

\end{tikzpicture}
    \caption{Tree partitioning in dimension $d=2$, with $n=6$ data points.}
    \label{figure1}
\end{figure}

The tree structure is usually data-dependent and indeed, it is in the construction itself that different trees differ. Of interest in the present paper is the CART program of \citet{BrFrOlSt84}, which will be described later on. Let us assume for now that we have at hand a regression tree $t_n$ (whose construction eventually depends upon the data $\mathscr D_n$), which takes constant values on each of $K\geq 2$ terminal nodes. It turns out that this estimate may be reinterpreted as a three-layer neural network estimate with two hidden layers and one output layer, as summarized in the following. Let $\mathscr H=\{H_1, \hdots, H_{K-1}\}$ be the collection of all hyperplanes participating in the construction of $t_n$. We note that each $H_k\in \mathscr H$ is of the form $H_k=\{\bx \in \mathds [0,1]^d:h_k(\bx)=0\}$, where $h_k(\bx)=x^{(j_k)}-\alpha_{j_k}$ for some (eventually data-dependent) $j_k \in \{1, \hdots, d\}$ and $\alpha_{j_k} \in [0,1]$. To reach the leaf of the query point $\bx$, we find, for each hyperplane $H_k$, the side on which $\bx$ falls ($+1$ codes for right and $-1$ for left). With this notation, the tree estimate $t_n$ is identical to the neural network described below. 

{\bf First hidden layer.}  
The first hidden layer of neurons corresponds to $K-1$ perceptrons (one for each inner tree node), whose activation is defined as 
$$
\tau(h_k(\bx)) = \tau(x^{(j_k)}-\alpha_{j_k}),
$$
where $\tau(u)=2\mathds 1_{u\geq 0}-1$ is a threshold activation function. The weight vector is merely a single one-hot vector for feature $j_k$, and $-\alpha_{j_k}$ is the bias value. So, for each split in the tree, there is a neuron in layer 1 whose activity encodes the relative position of an input $\bx$ with respect to the concerned split. In total, the first layer outputs the $\pm 1$-vector $(\tau(h_1(\bx)), \hdots, \tau(h_{K-1}(\bx)))$, which describes all decisions of the inner tree nodes (including nodes off the tree path of $\bx$). The quantity $\tau(h_k(\bx))$ is $+1$ if $\bx$ is on one side of the hyperplane $H_k$, $-1$ if $\bx$ is on the other side of $H_k$, and by convention $+1$ if $\bx \in H_k$. We again stress that each neuron $k$ of this layer is connected to one, and only one, input $x^{(j_k)}$, and that this connection has weight 1 and offset $-\alpha_{j_k}$. An example is presented in Figure \ref{figure2}. Given these particular activations of layer 1, layer 2 can then easily reconstruct the precise tree leaf membership (i.e., the terminal cell) of $\bx$. 
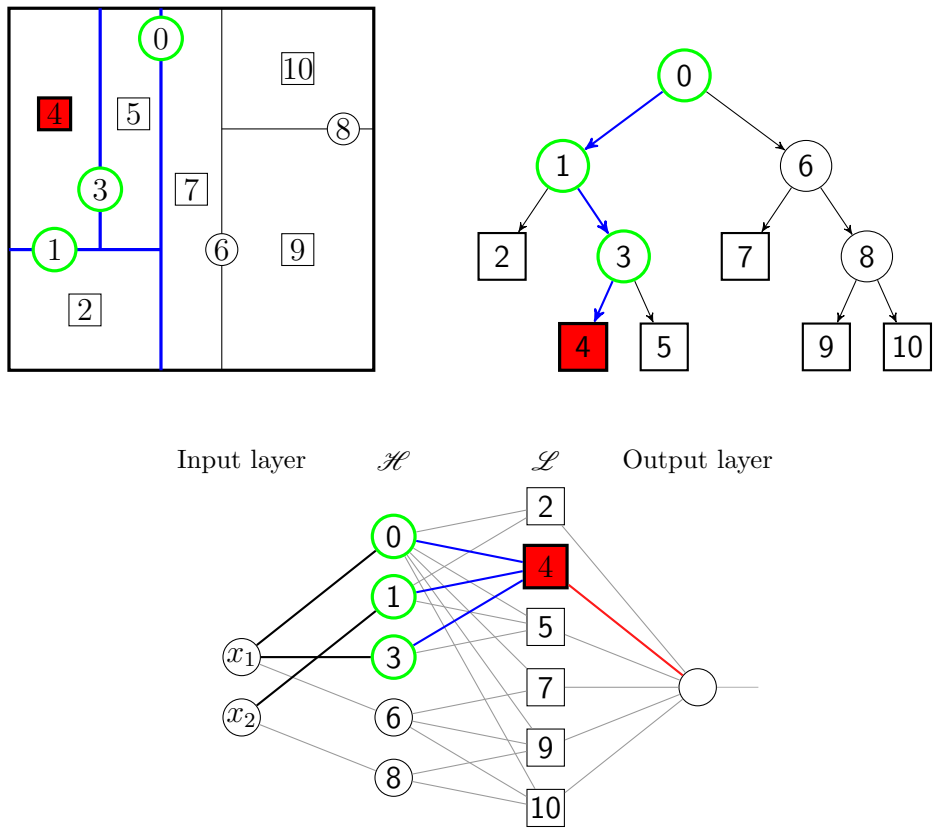
\begin{figure}[!h]
    \centering
\begin{tabular}{lcr}
\begin{tikzpicture}[scale=0.8]
 \tikzstyle{neuron}=[circle,  draw=black,
    fill=white, minimum size=12pt,  inner sep=0pt,  font=\sffamily]
     \tikzstyle{leaf neuron}=[rectangle,  draw=black,
    fill=white, minimum size=12pt,  inner sep=0pt,  font=\sffamily]
 \tikzstyle{branch neuron}=[neuron, draw=green, very thick, fill = white, minimum size = 16pt, inner sep = 2pt]
    
    \tikzstyle{special leaf neuron} = [rectangle, draw=black, very thick, fill = red, minimum size = 12pt, inner sep = 2pt]

\draw[very thick] (0,0) rectangle (6,6);
\draw[blue, very thick] (2.5,0) -- (2.5, 6);
\draw (3.5,0) -- (3.5, 6);
\draw[blue, very thick] (0,2) -- (2.5,2);
\draw[blue, very thick] (1.5,2) -- (1.5,6);
\draw (3.5,4) -- (6, 4);

 \node[branch neuron] (I-1) at (2.5,5.5) {$0$};
 \node[branch neuron] (I-2) at (0.75,2) {$1$};
 \node[branch neuron] (I-3) at (1.5,3) {$3$};
 \node[neuron] (I-4) at (3.5,2) {$6$};
 \node[neuron] (I-5) at (5.5,4) {$8$};

 \node[leaf neuron] (II-1) at (1.25,1) {$2$};
 \node[leaf neuron] (II-3) at (2.05,4.25) {$5$};
 \node[leaf neuron] (II-4) at (3,3) {$7$};
 \node[leaf neuron] (II-5) at (4.75,2) {$9$};
 \node[leaf neuron] (II-6) at (4.75,5) {$10$};

 \node[special leaf neuron] (II-1) at (0.75,4.25) {$4$};
\end{tikzpicture}
&  
\hspace{0.5cm}
&
   \begin{tikzpicture}[->,>=stealth',level/.style={sibling distance = 4cm/#1,
  level distance = 1.5cm}, scale = 0.8] 

\node [arn_b] {0}
    child[emph]{ node [arn_b] {1} 
            child[norm]{ node [arn_x] {2} 
            }	
            child{ node [arn_b] {3}
							child{ node [arn_Leaf] {4}}
							child[norm]{ node [arn_x] {5}}	}                            
    }
    child{ node [arn_n] {6}
            child{ node [arn_x] {7} }
            child{ node [arn_n] {8}
							child{ node [arn_x] {9}}
							child{ node [arn_x] {10}}
            }
		}
; 
\end{tikzpicture}

\end{tabular}

\def\layersep{2.5cm}
\def\HLOneN{5}
\def\HLTwoN{6}
\def\nFeats{2}
\vspace{0.7cm}

\begin{tikzpicture}[shorten >=0pt,-,draw=black, node distance=\layersep, scale = 0.8]
    \tikzstyle{every pin edge}=[-, draw = black!30]
    \tikzstyle{neuron}=[circle,  draw=black,
    fill=white, minimum size=14pt,  inner sep=0pt,  font=\sffamily]

    \tikzstyle{input neuron}=[neuron];
    \tikzstyle{branch neuron}=[neuron, draw=green, very thick, fill = white, minimum size = 16pt, inner sep = 2pt]
    \tikzstyle{invisible neuron}=[circle, draw = white]
    
    \tikzstyle{leaf neuron}=[rectangle,  draw=black,
    fill=white, minimum size=14pt,  inner sep=0pt,  font=\sffamily]

    \tikzstyle{special leaf neuron} = [rectangle, draw=black, very thick, fill = red, minimum size = 16pt, inner sep = 2pt]
    \tikzstyle{output neuron}=[neuron];
    \tikzstyle{hidden neuron}=[neuron];
    \tikzstyle{annot} = [text width=8em, text centered, font = \footnotesize]

    \foreach \name / \y in {1,...,\nFeats}
        \node[input neuron] (I-\name) at (0,-2-\y) {$x_{\y}$};



             \foreach \name / \y in {1}
        \path[yshift=0.5cm]
         node[branch neuron] (HL1-\name) at (\layersep,-\y cm-0.5 cm) { 0};
         
         \foreach \name / \y in {2}
        \path[yshift=0.5cm]
         node[branch neuron] (HL1-\name) at (\layersep,-\y cm-0.5 cm) { 1};

%

	\foreach \name / \y in {3}
        \path[yshift=0.5cm]
           node[branch neuron] (HL1-\name) at (\layersep,-\y cm -0.5 cm) {3};

    \foreach \name / \y in {4}
        \path[yshift=0.5cm]
           node[hidden neuron] (HL1-\name) at (\layersep,-\y cm -0.5 cm) {6};
    \foreach \name / \y in {5}
        \path[yshift=0.5cm]
           node[hidden neuron] (HL1-\name) at (\layersep,-\y cm -0.5 cm) {8};


  \foreach \name / \y in {1}
        \path[yshift=0.5cm]
            node[leaf neuron] (HL2-\name) at (2*\layersep,-\y cm) {2};
            
	
	\foreach \name / \y in {2}
        \path[yshift=0.5cm]
            node[special leaf neuron] (HL2-\name) at (2*\layersep,-\y cm) {4};
     
    \foreach \name / \y in {3}
        \path[yshift=0.5cm]
            node[leaf neuron] (HL2-\name) at (2*\layersep,-\y cm) {5};
 \foreach \name / \y in {4}
        \path[yshift=0.5cm]
            node[leaf neuron] (HL2-\name) at (2*\layersep,-\y cm) {7};
 \foreach \name / \y in {5}
        \path[yshift=0.5cm]
            node[leaf neuron] (HL2-\name) at (2*\layersep,-\y cm) {9};
 \foreach \name / \y in {6}
        \path[yshift=0.5cm]
            node[leaf neuron] (HL2-\name) at (2*\layersep,-\y cm) {10};

     \node[output neuron,pin={[pin edge={-}]right:}] (O1) at (3*\layersep,-3.5 cm)  {};

   
    \path (I-1) edge (HL1-1)[black, thick];
    \path (I-1) edge (HL1-3)[black, thick];
    \path (I-1) edge (HL1-4)[black!40];
    \path (I-2) edge (HL1-2)[black, thick];
    \path (I-2) edge (HL1-5)[black!40];



\foreach \source in {1,2}
	 \path (HL1-\source) edge (HL2-1)[black!40];
\foreach \source in {1,2,3}
	 \path (HL1-\source) edge (HL2-3)[black!40];
\foreach \source in {1,4}
	 \path (HL1-\source) edge (HL2-4)[black!40];
\foreach \source in {1,4,5}
	 \path (HL1-\source) edge (HL2-5)[black!40];
\foreach \source in {1,4,5}
	 \path (HL1-\source) edge (HL2-6)[black!40];

   \foreach \source in {1,2,3}
	 \path (HL1-\source) edge (HL2-2)[blue, thick];

    \foreach \source in {1,...,\HLTwoN}
        \path (HL2-\source) edge (O1)[black!40];

    \path (HL2-2) edge (O1)[red!90, thick];

    \node[annot,above of=HL1-1, node distance=1cm] (hl1) {$\mathscr{H}$};
    \node[annot,left of=hl1, node distance=2cm] {Input layer};
    \node[annot,right of=hl1, node distance=2cm](hl2){ $\mathscr{L}$};
    \node[annot,right of=hl2, node distance=2cm](hl2){Output layer};

\end{tikzpicture}

    \caption{An example of regression tree ({\bf top}) and the corresponding neural network ({\bf down}).}
    \label{figure2}
\end{figure}

{\bf Second hidden layer.} Layer 1 outputs a $(K-1)$-dimensional vector of $\pm 1$-bits that encodes for the precise location of $\bx$ in the leaves of the tree. The leaf node identity of $\bx$ can now be extracted from this vector using a weighted combination of the bits, together with an appropriate thresholding. 

Let $\mathscr L=\{L_1, \hdots, L_K\}$ be the collection of all tree leaves, and let $L(\bx)$ be the leaf containing $\bx$. The second hidden layer has $K$ neurons, one for each leaf, and assigns a terminal cell to $\bx$ as explained below.  We connect a unit $k$ from layer 1 to a unit $k'$ from layer 2 if and only if the hyperplane $H_k$ is involved in the sequence of splits forming the path from the root to the leaf $L_{k'}$. The connection has weight $+1$ if, in that path, the split by $H_k$ is from a node to a right child, and $-1$ otherwise.  So, if $(u_1(\bx), \hdots, u_{K-1}(\bx))$ is the vector of $\pm 1$-bits seen at the output of layer 1, the output $v_{k'}(\bx)\in \{-1,1\}$ of neuron $k'$ is $\tau(\sum_{k \to k' } b_{k,k'}u_k(\bx)+b^{0}_{k'})$, where notation $k \to k'$ means that $k$ is connected to $k'$ and $b_{k,k'} = \pm 1$ is the corresponding weight. The offset $b^0_{k'}$ is set to 
\begin{equation}
\label{offset}
b^0_{k'}=-\ell (k')+\frac{1}{2},
\end{equation}
where $\ell(k')$ is the length of the path from the root to $L_{k'}$. To understand the rationale behind the choice (\ref{offset}), observe that there are exactly $\ell(k')$ connections starting from the first layer and pointing to $k'$, and that 
\begin{align}
    \left\lbrace
\begin{array}{ll}
\sum_{k \to k' } b_{k,k'}u_k(\bx) -\ell(k') +  \frac{1}{2} = \frac{1}{2}  &  \textrm{if}~ \bx \in L_{k'} \\
\sum_{k \to k' } b_{k,k'}u_k(\bx) -\ell(k') +  \frac{1}{2}  \leq -\frac{1}{2}  & \textrm{otherwise.}
     \end{array}
     \right. \label{def_threshold_function}
     \end{align}
Thus, with the choice (\ref{offset}), the argument of the threshold function is 1/2 if $\bx \in L_{k'}$ and is smaller than $-1/2$ otherwise. Hence $v_{k'}(\bx) = 1$ if and only if the terminal cell of $\bx$ is $L_{k'}$. To summarize, the second hidden layer outputs a vector of $\pm 1$-bits $(v_1(\bx), \hdots, v_{K}(\bx))$ whose components equal $-1$ except the one corresponding to the leaf $L(\bx)$, which is $+1$. 

{\bf Output layer.} Let $(v_1(\bx), \hdots, v_{K}(\bx))$ be the output of the second hidden layer. If $v_{k'}(\bx)=1$, then the output layer computes the average $\bar Y_{k'}$ of the $Y_i$ corresponding to $\bX_i$ falling in $L_{k'}$. This is equivalent to take
\begin{equation}
\label{NFI}
t_n(\bx)=\sum_{k'=1}^{K} w_{k'} v_{k'}(\bx)+{b}_{\textrm{out}},
\end{equation}
where $w_{k'}=\frac{1}{2}\bar Y_{k'}$ for all $k'\in \{1, \hdots, K\}$, and ${b}_{\textrm{out}}=\frac{1}{2}\sum_{k'=1}^K \bar Y_{k'}$.
\subsection{The CART program}
We know from the preceding section that every regression tree may be seen as a neural network estimate with two hidden layers and threshold activation functions. Let us now be more precise about the details of the CART program of \citet{BrFrOlSt84} to induce a regression tree. The core of their approach is a tree with $K_n$ leaf regions defined by a partition of the space based on the $n$ data points. When constructing the tree, the so-called CART-split criterion is applied recursively. This criterion determines which input direction should be used for the split and where the cut should be made. Let $A$ be a generic cell and denote by $N_n(A)$ the number of examples falling in $A$. Formally, a cut in $A$ is a pair $(j,\alpha)$, where $j$ is a dimension from $\{1, \hdots, d\}$ and $\alpha\in[0,1]$ is the position of the cut along the $j$-th coordinate, within the limits of $A$. Let $\mathscr{C}_A$ be the set of all such possible cuts in $A$. Then, using the notation $\bX_i = (\bX_i^{(1)}, \hdots, \bX_i^{(d)} )$, the CART-split criterion takes the form, for any $(j,\alpha) \in \mathscr{C}_A$,
\begin{align}
L_{n}(j,\alpha) = & \frac{1}{N_n(A)} \sum_{i=1}^n (Y_i - \bar{Y}_{A})^2\mathds{1}_{\bX_i \in A} \nonumber \\
&  - \frac{1}{N_n(A)} \sum_{i=1}^n (Y_i - \bar{Y}_{A_{L}} \mathds{1}_{\bX_i^{(j)}  < \alpha} - \bar{Y}_{A_{R}} \mathds{1}_{\bX_i^{(j)} \geq \alpha})^2 \mathds{1}_{\bX_i \in A},  \label{chapitre0_definition_empirical_CART_criterion}
\end{align}
where $A_L = \{ \bx \in A: \bx^{(j)} < \alpha\}$, $A_R = \{ \bx \in A: \bx^{(j)} \geq \alpha\}$, and $\bar{Y}_{A}$ (resp., $\bar{Y}_{A_{L}}$, $\bar{Y}_{A_{R}}$) is the average of the $Y_i$ belonging to $A$ (resp., $A_{L}$, $A_{R}$), with the convention $0/0=0$. The quantity (\ref{chapitre0_definition_empirical_CART_criterion}) measures the (renormalized) difference between the empirical variance in the node before and after a cut is performed. For each cell $A$, the best cut $(j_n^{\star},\alpha_n^{\star})$ is selected by maximizing $L_n(j,\alpha)$ over $\mathscr{C}_A$; that is,
\begin{equation}
\label{un}
(j_n^{\star},\alpha_n^{\star}) \in \argmax_{(j,\alpha) \in \mathscr{C}_A } L_{n}(j,\alpha).
\end{equation}
(To remove some of the ties in the argmax, the best cut is always performed in the middle of two consecutive data points.) So, at each cell, the algorithm evaluates criterion (\ref{chapitre0_definition_empirical_CART_criterion}) over all possible cuts in the $d$ directions and returns the best one. This process is applied recursively down the tree, and stops when the tree contains exactly $K_n$ terminal nodes, where $K_n\geq 2$ is an integer eventually depending on $n$.
\subsection{Random forests}
A random forest is a predictor consisting of a collection of $M$ (large) randomized CART-type regression trees. For the $m$-th tree in the family, the predicted value at the query point $\bx$ is denoted by $t(\bx\,; \Theta_{m},\mathscr D_n)$, where $\Theta_1, \hdots,\Theta_M$ are random variables, distributed the same as a generic random variable $\Theta$. It is assumed that $\Theta_1, \hdots, \Theta_M$, $\Theta$, and $\mathscr D_n$ are mutually independent.
The variable $\Theta$, which models the extra randomness introduced in each tree construction, is used to $(i)$ resample the training set prior to the growing of individual trees, and $(ii)$ select the successive directions for splitting via a randomized version of the CART criterion---see below. Lastly, the trees are combined to form the forest estimate 
\begin{equation*}
t(\bx\,; \Theta_1, \hdots, \Theta_M, \mathscr{D}_n)=\frac{1}{M}\sum_{m=1}^M t(\bx\,; \Theta_m,\mathscr D_n).\label{chapitre0_finite_forest}
\end{equation*}
To lighten notation we write $t_{M,n}(\bx)$ instead of $t(\bx \,; \Theta_1, \hdots, \Theta_M, \mathscr{D}_n)$.

The method works by growing the $M$ randomized trees as follows. Let $a_n\geq 2$ be an integer smaller than or equal to $n$. Prior to the construction of each tree, $a_n$ observations are drawn at random without replacement from the original data set; then, at each cell of the current tree, a split is performed by choosing uniformly at random, without replacement, a subset $\Mtry \subseteq \{1,\hdots,d\}$ of cardinality $m_{\textrm{try}}:=|\Mtry|$, by evaluating criterion (\ref{chapitre0_definition_empirical_CART_criterion}) over all possible cuts in the $m_{\textrm{try}}$ directions, and by returning the best one. In other words, for each cell $A$, the best cut $(j_n^{\star},\alpha_n^{\star})$ is selected by maximizing $L_n(j,\alpha)$ over $\mathscr{M}_{\textrm{\scriptsize try}}$ and $\mathscr{C}_A$; that is,
\begin{equation}
\label{deux}
(j_n^{\star},\alpha_n^{\star}) \in \argmax\limits_{\substack{j \in \mathscr{M}_{ \textrm{\tiny try}}\\(j,\alpha) \in \mathscr{C}_A }} L_{n}(j,\alpha).
\end{equation}
The essential difference between (\ref{un}) and (\ref{deux}) is that (\ref{deux}) is evaluated over a subset  $\Mtry$ of randomly selected coordinates, {\it not} over the whole range $\{1, \hdots, d\}$. The parameter $m_{\textrm{try}}$, which aims at reducing the computational burden and creating some diversity between the trees, is independent of $n$ and often set to $d/3$. Also, because of the without-replacement sampling, each tree is constructed on a subset of $a_n$ examples picked within the initial sample, {\it not} on the whole sample $\mathscr D_n$. In accordance with the CART program, the construction of individual trees is stopped when each tree reaches exactly $K_n$ terminal nodes (recall that $K_n \in \{2, \hdots,a_n\}$ is a parameter of the algorithm). We insist on the fact that the number of leaves in each tree equals $K_n$. There are variants of the approach in terms of parameter choices and resampling mode that will not be explored here---see, e.g., the discussion in \citet{BiSc15} and the comments after Theorem \ref{maintheo1}.

Finally, by following the principles of Subsection 2.1, each random tree estimate $t(\cdot \, ;\Theta_m,\mathscr D_n)$, $1\leq m\leq M$, of the forest can be reinterpreted in the setting of neural networks. The $M$ resulting networks are different because they correspond to different random trees. 
We see in particular that the $m$-th network has exactly $K_n-1$ neurons in the first hidden layer and $K_n$ in the second one. We also note that the 
network architecture (i.e., the way neurons are connected) and the associated coefficients depend on both $\Theta_m$ {\it and} $\mathscr D_n$. 
\section{Neural forests}
Consider the $m$-th tree estimate $t(\cdot \, ;\Theta_m,\mathscr D_n)$ of the forest, seen as a neural network estimate. Conditional on $\Theta_m$ and $\mathscr D_n$, the architecture of this network is fixed, and so are the weights and offsets of the three layers. A natural idea is then to keep the structure of the network intact and let the parameters vary in a subsequent network training procedure with backpropagation training. In other words, once the connections between the neurons have been designed by the tree-to-network mapping, we could then learn even better network parameters by minimizing some empirical mean squared error for this network over the sample $\mathscr D_n$. This additional training can potentially improve the predictions of the original random forest, and we will see more about this later in the experiments.

To allow for training based on gradient backpropagation, the activation functions must be differentiable. A natural idea is to
replace the original relay-type activation function $\tau(u)=2\mathds 1_{u\geq 0}-1$ with a smooth approximation of it; for this the hyperbolic tangent activation function 
$$\sigma(u) := \tanh(u) = \frac{e^{u}-e^{-u}}{e^{u}+e^{-u}}=\frac{e^{2u}-1}{e^{2u}+1},$$
which has a range from $-1$ to $1$ is chosen. More precisely, we use $\sigma_1(u)=\sigma(\gamma_1 u)$ at every neuron of the first hidden layer and $\sigma_2(u)=\sigma(\gamma_2 u)$ at every neuron of the second one. Here, $\gamma_1$ and $\gamma_2$ are positive hyperparameters that determine the contrast of the hyperbolic tangent activation: the larger $\gamma_1$ and $\gamma_2$, the sharper the transition from $-1$ to $1$. Of course, as $\gamma_1$ and $\gamma_2$ approach infinity, the continuous functions $\sigma_1$ and $\sigma_2$ converge to the threshold function. Besides eventually providing better generalization, the hyperbolic tangent activation functions favor smoother decision boundaries and permit a relaxation of crisp tree node membership. Lastly, they allow to operate with a smooth approximation of the discontinuous step activation function. This makes the network loss function differentiable with respect to the parameters everywhere, and gradients can be backpropagated to train the network. Similar ideas are developed in the so-called soft tree models \citep{Jo94,OlWe03,GeWe05,YiAl13}.

In this sparse setting, the number of parameters is much smaller than in a fully connected feed-forward network with the same number of neurons. Recall that the two hidden layers have respectively $K_n-1$ and $K_n$ hidden nodes. Without any information regarding the connections between neurons (fully connected network), fitting such a network would require optimizing a total of $(d+1)(K_n-1)+K_n^2+K_n+1$ parameters, which is $\mathscr O(dK_n+K_n^2)$. On the other hand, assuming that the tree generated by the CART algorithm is roughly balanced, then the average depth of the tree is $\mathscr O(\log K_n)$. This gives, on average, $2(K_n-1)+\mathscr O(K_n\log K_n)+K_n+1$ parameters to fit, which is $\mathscr O(K_n\log K_n)$. For large $K_n$, this quantity can be much smaller than $\mathscr O(dK_n+K_n^2)$ and, in any case, it is independent of the dimension $d$---a major computational advantage in high-dimensional settings. If an unbalanced tree does occur, then it has at most $\mathscr O(K_n)$ levels and the total work required is $\mathscr O(K_n^2)$. Thus, even in this extreme case, our algorithm is independent of $d$ and should be competitive in high dimensions, when $d\gg K_n$.
\begin{rem}
Training a network with sparse connectivity retains some degree of interpretability of its internal representation and may be seen as a relaxation of original tree structures.
But besides these sparse networks, the relaxation of tree structures can go even further when allowing for full connectivity between layers---in this case, the tree structure is merely used as initialization for a fully connected network of the same size. In this setting, all weights belonging to tree structures have a nonzero initialization value, while the other weights start with 0. During training then, all weights can be modified so that arbitrary connections between the layers (i.e., also across trees) can be learned. The initial tree-type parametrization provides a strong inductive bias, compared to a random initialization, which contains valuable information and mimics the regression function of a CART-type tree already before backpropagation training. Compared to a random initialization, this gives the network an effective warm start.
The effects of using random forests as warm starts for neural networks will be further explored in the experimental Section 5.
\end{rem}
Thus, the above mechanisms allow to convert the $M$ CART-type trees into $M$ tree-type neural networks that have different sparse architectures and smooth activation functions at the nodes. Interestingly, this link opens up a principled way for inferring the structure of a neural network, by viewing a regression tree as a specially structured network and putting some kind of prior information about the weights.

To complete the presentation, it remains to explain how to combine the $M$ individual networks. This can be done in at least two different ways, which are described below. We call the two resulting estimates {\it neural forests}.

{\bf Method 1: Independent training.} The parameters of each tree-type network are fitted network by network, independently of each other. With this approach, we end up with an ensemble of $M$ ``small'' neural network estimates $r(\cdot \, ;\Theta_m,\mathscr D_n)$, $1\leq m \leq M$, which are finally averaged to form the estimate 
\begin{equation}
\label{average}
r(\bx\,; \Theta_1, \hdots, \Theta_M, \mathscr{D}_n)=\frac{1}{M}\sum_{m=1}^M r(\bx\,; \Theta_m,\mathscr D_n)
\end{equation}
(see the illustration in Figure \ref{fig_method1}). To lighten notation we write $r_{M,n}(\bx)$ instead of the more complicated $r(\bx \,; \Theta_1, \hdots, \Theta_M, \mathscr{D}_n)$, keeping in mind that $r_{M,n}(\bx)$ depends on both $\Theta_1, \hdots, \Theta_M$ and the sample $\mathscr D_n$. 

The minimization program implemented at each ``small'' network is  described in Section 4 below, together with the statistical properties of $r_{M,n}(\bx)$. This predictor has the flavor of a randomized ensemble (forest) of neural networks. 

\begin{figure}
\def\layersep{2.5cm}
\def\HLOneN{5}
\def\HLTwoN{6}
\def\nFeats{2}

\begin{tikzpicture}[shorten >=0pt,-,draw=black, node distance=\layersep]
    \tikzstyle{every pin edge}=[-, draw = black!30]
    \tikzstyle{neuron}=[circle,  draw=black,
    fill=white, minimum size=14pt,  inner sep=0pt,  font=\sffamily]

    \tikzstyle{input neuron}=[neuron];
    \tikzstyle{branch neuron}=[neuron, draw=green, very thick, fill = white, minimum size = 16pt, inner sep = 2pt]
    \tikzstyle{invisible neuron}=[circle, draw = white]
    
    \tikzstyle{leaf neuron}=[rectangle,  draw=black,
    fill=white, minimum size=14pt,  inner sep=0pt,  font=\sffamily]

    \tikzstyle{special leaf neuron} = [rectangle, draw=black, very thick, fill = red, minimum size = 16pt, inner sep = 2pt]
    \tikzstyle{output neuron}=[circle,  draw=red,
    fill=white, minimum size=14pt,  inner sep=0pt,  font=\sffamily];
        \tikzstyle{output neuron bis}==[rectangle,  draw=black,
    fill=white, minimum size=14pt,  inner sep=4pt,  font=\sffamily];

    \tikzstyle{hidden neuron}=[neuron];
    \tikzstyle{annot} = [text width=8em, text centered, font = \footnotesize]

    \foreach \name / \y in {1}
        \node[input neuron] (I-\name) at (\layersep,-2*\y) {$x_{\y}$};
    \foreach \name / \y in {2}
        \node[input neuron] (I-\name) at (\layersep,-3.9) {$x_{\y}$};


    \foreach \name / \y in {1}
        \path[yshift=0.5cm]
         node[branch neuron] (HL1-\name) at (2*\layersep,-\y cm-0.5 cm) { 0};
    
    \foreach \name / \y in {2}
        \path[yshift=0.5cm]
         node[branch neuron] (HL1-\name) at (2*\layersep,-\y cm-0.5 cm) { 1};
    

	\foreach \name / \y in {3}
        \path[yshift=0.5cm]
           node[branch neuron] (HL1-\name) at (2*\layersep,-\y cm -0.5 cm) {3};
           
    \foreach \name / \y in {4}
        \path[yshift=0.5cm]
           node[hidden neuron] (HL1-\name) at (2*\layersep,-\y cm -0.5 cm) {6};
    \foreach \name / \y in {5}
        \path[yshift=0.5cm]
           node[hidden neuron] (HL1-\name) at (2*\layersep,-\y cm -0.5 cm) {8};


  \foreach \name / \y in {1}
        \path[yshift=0.5cm]
            node[leaf neuron] (HL2-\name) at (3*\layersep,-\y cm) {2};
            
	
	\foreach \name / \y in {2}
        \path[yshift=0.5cm]
            node[special leaf neuron] (HL2-\name) at (3*\layersep,-\y cm) {4};
     
    \foreach \name / \y in {3}
        \path[yshift=0.5cm]
            node[leaf neuron] (HL2-\name) at (3*\layersep,-\y cm) {5};
 \foreach \name / \y in {4}
        \path[yshift=0.5cm]
            node[leaf neuron] (HL2-\name) at (3*\layersep,-\y cm) {7};
 \foreach \name / \y in {5}
        \path[yshift=0.5cm]
            node[leaf neuron] (HL2-\name) at (3*\layersep,-\y cm) {9};
 \foreach \name / \y in {6}
        \path[yshift=0.5cm]
            node[leaf neuron] (HL2-\name) at (3*\layersep,-\y cm) {10};

     \node[output neuron] (O1) at (4*\layersep,-3 cm)  {};

   
    \path (I-1) edge (HL1-1)[black, thick];
    \path (I-1) edge (HL1-3)[black, thick];
    \path (I-1) edge (HL1-4)[black!40];
    \path (I-2) edge (HL1-2)[black, thick];
    \path (I-2) edge (HL1-5)[black!40];



\foreach \source in {1,2}
	 \path (HL1-\source) edge (HL2-1)[black!40];
\foreach \source in {1,2,3}
	 \path (HL1-\source) edge (HL2-3)[black!40];
\foreach \source in {1,4}
	 \path (HL1-\source) edge (HL2-4)[black!40];
\foreach \source in {1,4,5}
	 \path (HL1-\source) edge (HL2-5)[black!40];
\foreach \source in {1,4,5}
	 \path (HL1-\source) edge (HL2-6)[black!40];

   \foreach \source in {1,2,3}
	 \path (HL1-\source) edge (HL2-2)[blue, thick];

    \foreach \source in {1,...,\HLTwoN}
        \path (HL2-\source) edge (O1)[black!40];

    \path (HL2-2) edge (O1)[red!90, thick];



\foreach \name / \y in {1,...,\nFeats}
        \node[input neuron] (Ib-\name) at (\layersep,-11.5-2*\y) {$x_{\y}$};

 \foreach \name / \y in {1}
        \path[yshift=0.5cm]
         node[branch neuron] (HL1b-\name) at (2*\layersep,-12- \y) { 0};

\foreach \name / \y in {2}
        \path[yshift=0.5cm]
         node[branch neuron] (HL1b-\name) at (2*\layersep,-12- \y) { 1};

	\foreach \name / \y in {3}
        \path[yshift=0.5cm]
           node[branch neuron] (HL1b-\name) at (2*\layersep,-12 -\y ) {2};
           
    \foreach \name / \y in {4}
        \path[yshift=0.5cm]
           node[hidden neuron] (HL1b-\name) at (2*\layersep,-12 -\y) {3};
    \foreach \name / \y in {5}
        \path[yshift=0.5cm]
           node[hidden neuron] (HL1b-\name) at (2*\layersep,-12 -\y) {7};

\foreach \name / \y in {1}
        \path[yshift=0.5cm]
            node[leaf neuron] (HL2b-\name) at (3*\layersep,-11.5 -\y ) {4};
         
	\foreach \name / \y in {2}
        \path[yshift=0.5cm]
            node[leaf neuron] (HL2b-\name) at (3*\layersep,-11.5-\y ) {5};
     
    \foreach \name / \y in {3}
        \path[yshift=0.5cm]
            node[special leaf neuron] (HL2b-\name) at (3*\layersep,-11.5-\y ) {6};
 \foreach \name / \y in {4}
        \path[yshift=0.5cm]
            node[leaf neuron] (HL2b-\name) at (3*\layersep,-11.5-\y ) {8};
 \foreach \name / \y in {5}
        \path[yshift=0.5cm]
            node[leaf neuron] (HL2b-\name) at (3*\layersep,-11.5-\y ) {9};
 \foreach \name / \y in {6}
        \path[yshift=0.5cm]
            node[leaf neuron] (HL2b-\name) at (3*\layersep,-11.5-\y ) {10};

     \node[output neuron] (O1b) at (4*\layersep,-14.5)  {};

    \path (Ib-1) edge (HL1b-1)[black, thick];
    \path (Ib-1) edge (HL1b-3)[black, thick];
    \path (Ib-2) edge (HL1b-4)[black!40];
    \path (Ib-2) edge (HL1b-2)[black, thick];
    \path (Ib-2) edge (HL1b-5)[black!40];

\foreach \source in {1,2}
	 \path (HL1b-\source) edge (HL2b-1)[black!40];
\foreach \source in {1,2,3}
	 \path (HL1b-\source) edge (HL2b-3)[black!40];
\foreach \source in {1,4}
	 \path (HL1b-\source) edge (HL2b-4)[black!40];
\foreach \source in {1,4,5}
	 \path (HL1b-\source) edge (HL2b-5)[black!40];
\foreach \source in {1,4,5}
	 \path (HL1b-\source) edge (HL2b-6)[black!40];

   \foreach \source in {1,2,3}
	 \path (HL1b-\source) edge (HL2b-3)[blue, thick];

    \foreach \source in {1,...,\HLTwoN}
        \path (HL2b-\source) edge (O1b)[black!40];

    \path (HL2b-3) edge (O1b)[red!90, thick];

    \node[annot,above of=HL1-1, node distance=2cm] (hl1) { $\mathscr{H}$};
   \node[annot,left of=hl1, node distance=2.5cm] {Input layer};
    \node[annot,right of=hl1](hl2){ $\mathscr{L}$};

    \node[annot,right of=hl2, node distance=2.5cm](hl3){Output layer};

     \node[output neuron bis,pin={[pin edge={-}]right:}] (O1c) at (5*\layersep,-9)  {Average};

     \node[output neuron,pin={[pin edge={-}]left:}] (O1d) at (4*\layersep,-7 cm)  {};
     \node[output neuron,pin={[pin edge={-}]left:}] (O1e) at (4*\layersep,-11 cm)  {};
     \node[output neuron,pin={[pin edge={-}]left:}] (O1f) at (4*\layersep,-9 cm)  {};

     \node[invisible neuron] (HOd) at (\layersep,-7 cm)  {$\hdots$};
     \node[invisible neuron] (H1d) at (2*\layersep,-7 cm)  {$\hdots$};
     \node[invisible neuron] (H2d) at (3*\layersep,-7 cm)  {$\hdots$};

     \node[invisible neuron] (HOe) at (\layersep,-11 cm)  {$\hdots$};
     \node[invisible neuron] (H1e) at (2*\layersep,-11 cm)  {$\hdots$};
     \node[invisible neuron] (H2e) at (3*\layersep,-11 cm)  {$\hdots$};

     \node[invisible neuron] (HOf) at (\layersep,-9 cm)  {$\hdots$};
     \node[invisible neuron] (H1f) at (2*\layersep,-9 cm)  {$\hdots$};
     \node[invisible neuron] (H2f) at (3*\layersep,-9 cm)  {$\hdots$};

	 \path (O1d) edge (O1c)[black, thick, dotted];
	 \path (O1e) edge (O1c)[black, thick, dotted];
	 \path (O1f) edge (O1c)[black, thick, dotted];
	 \path (O1) edge (O1c)[black, thick, dotted];
	 \path (O1b) edge (O1c)[black, thick, dotted];

\end{tikzpicture}
\caption{Method 1: Independent training.}
\label{fig_method1}
\end{figure}
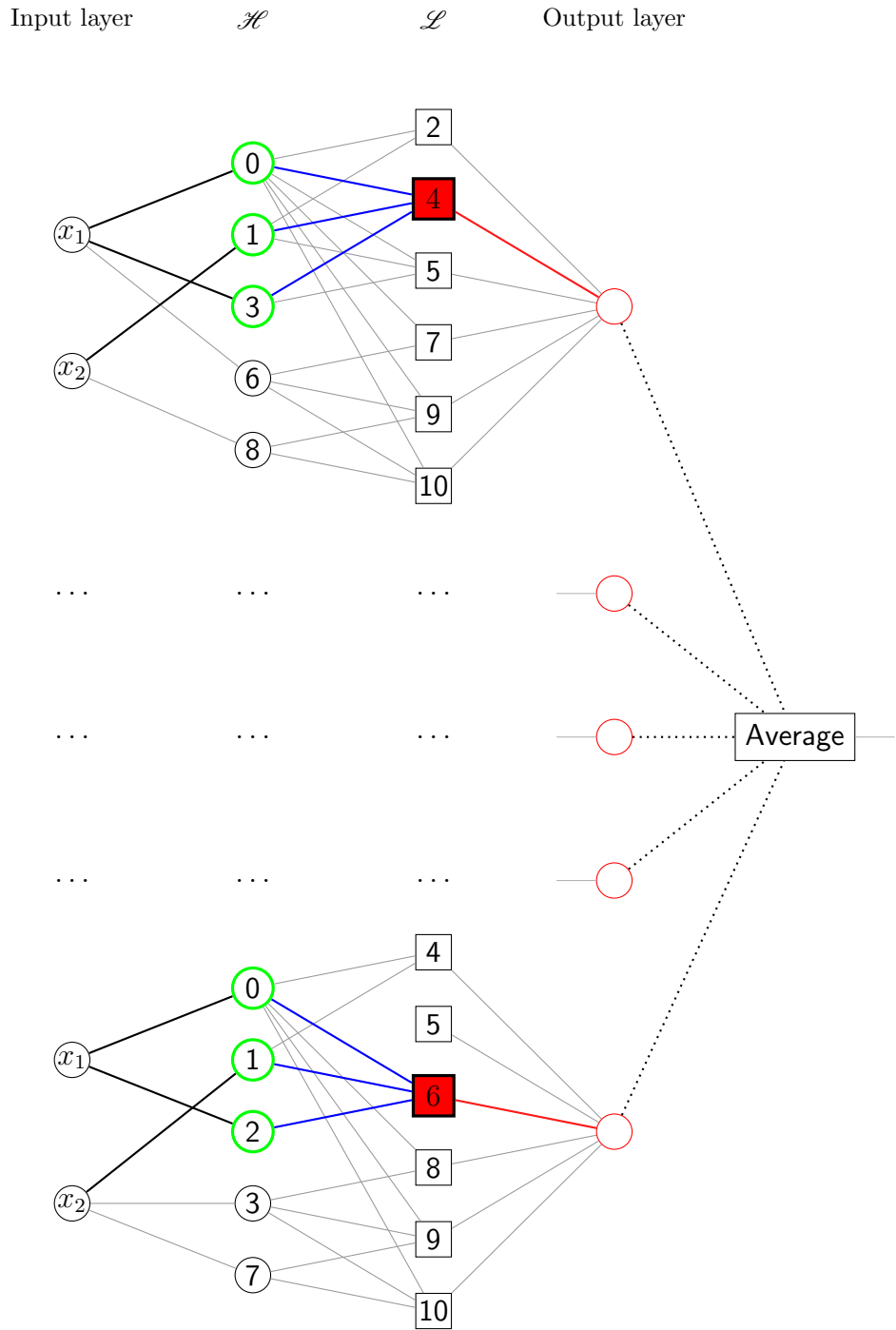
{\bf Method 2: Joint training.} In this approach, the individual tree networks are first concatenated into one single ``big'' network, as shown in Figure \ref{fig_method2}. The parameters of the resulting ``big'' network are then fitted jointly in one optimization procedure over the whole network.
\begin{figure}
\hspace{-2cm}
\centering
\def\layersep{2.5cm}
\def\HLOneN{5}
\def\HLTwoN{6}
\def\nFeats{2}

\begin{tikzpicture}[shorten >=0pt,-,draw=black, node distance=\layersep]
    \tikzstyle{every pin edge}=[-, draw = black!30]
    \tikzstyle{neuron}=[circle,  draw=black,
    fill=white, minimum size=14pt,  inner sep=0pt,  font=\sffamily]

    \tikzstyle{input neuron}=[neuron];
    \tikzstyle{branch neuron}=[neuron, draw=green, very thick, fill = white, minimum size = 16pt, inner sep = 2pt]
    \tikzstyle{invisible neuron}=[circle, draw = white]
    
    \tikzstyle{leaf neuron}=[rectangle,  draw=black,
    fill=white, minimum size=14pt,  inner sep=0pt,  font=\sffamily]

    \tikzstyle{special leaf neuron} = [rectangle, draw=black, very thick, fill = red, minimum size = 16pt, inner sep = 2pt]
    \tikzstyle{output neuron}=[neuron];
    \tikzstyle{hidden neuron}=[neuron];
    \tikzstyle{annot} = [text width=8em, text centered, font = \footnotesize]

    \foreach \name / \y in {1,...,\nFeats}
        \node[input neuron] (I-\name) at (\layersep,-6-2*\y) {$x_{\y}$};


    \foreach \name / \y in {1}
        \path[yshift=0.5cm]
         node[branch neuron] (HL1-\name) at (2*\layersep,-\y cm-0.5 cm) { 0};
         
       \foreach \name / \y in {2}
        \path[yshift=0.5cm]
         node[branch neuron] (HL1-\name) at (2*\layersep,-\y cm-0.5 cm) { 1};

	\foreach \name / \y in {3}
        \path[yshift=0.5cm]
           node[branch neuron] (HL1-\name) at (2*\layersep,-\y cm -0.5 cm) {3};
           
    \foreach \name / \y in {4}
        \path[yshift=0.5cm]
           node[hidden neuron] (HL1-\name) at (2*\layersep,-\y cm -0.5 cm) {6};
    \foreach \name / \y in {5}
        \path[yshift=0.5cm]
           node[hidden neuron] (HL1-\name) at (2*\layersep,-\y cm -0.5 cm) {8};


  \foreach \name / \y in {1}
        \path[yshift=0.5cm]
            node[leaf neuron] (HL2-\name) at (3*\layersep,-\y cm) {2};
            
	
	\foreach \name / \y in {2}
        \path[yshift=0.5cm]
            node[special leaf neuron] (HL2-\name) at (3*\layersep,-\y cm) {4};
     
    \foreach \name / \y in {3}
        \path[yshift=0.5cm]
            node[leaf neuron] (HL2-\name) at (3*\layersep,-\y cm) {5};
 \foreach \name / \y in {4}
        \path[yshift=0.5cm]
            node[leaf neuron] (HL2-\name) at (3*\layersep,-\y cm) {7};
 \foreach \name / \y in {5}
        \path[yshift=0.5cm]
            node[leaf neuron] (HL2-\name) at (3*\layersep,-\y cm) {9};
 \foreach \name / \y in {6}
        \path[yshift=0.5cm]
            node[leaf neuron] (HL2-\name) at (3*\layersep,-\y cm) {10};

     \node[output neuron, pin={[pin edge={-}]right:}] (O1) at (5*\layersep,-9 cm)  {};

   
    \path (I-1) edge (HL1-1)[black, thick];
    \path (I-1) edge (HL1-3)[black, thick];
    \path (I-1) edge (HL1-4)[black!40];
    \path (I-2) edge (HL1-2)[black, thick];
    \path (I-2) edge (HL1-5)[black!40];



\foreach \source in {1,2}
	 \path (HL1-\source) edge (HL2-1)[black!40];
\foreach \source in {1,2,3}
	 \path (HL1-\source) edge (HL2-3)[black!40];
\foreach \source in {1,4}
	 \path (HL1-\source) edge (HL2-4)[black!40];
\foreach \source in {1,4,5}
	 \path (HL1-\source) edge (HL2-5)[black!40];
\foreach \source in {1,4,5}
	 \path (HL1-\source) edge (HL2-6)[black!40];

   \foreach \source in {1,2,3}
	 \path (HL1-\source) edge (HL2-2)[blue, thick];

    \foreach \source in {1,...,\HLTwoN}
        \path (HL2-\source) edge (O1)[black!40];

    \path (HL2-2) edge (O1)[red!90, thick];



 \foreach \name / \y in {1}
        \path[yshift=0.5cm]
         node[branch neuron] (HL1b-\name) at (2*\layersep,-12- \y) { 0};
 
 \foreach \name / \y in {2}
        \path[yshift=0.5cm]
         node[branch neuron] (HL1b-\name) at (2*\layersep,-12- \y) { 1};

	\foreach \name / \y in {3}
        \path[yshift=0.5cm]
           node[branch neuron] (HL1b-\name) at (2*\layersep,-12 -\y ) {2};
           
    \foreach \name / \y in {4}
        \path[yshift=0.5cm]
           node[hidden neuron] (HL1b-\name) at (2*\layersep,-12 -\y) {3};
    \foreach \name / \y in {5}
        \path[yshift=0.5cm]
           node[hidden neuron] (HL1b-\name) at (2*\layersep,-12 -\y) {7};

\foreach \name / \y in {1}
        \path[yshift=0.5cm]
            node[leaf neuron] (HL2b-\name) at (3*\layersep,-11.5 -\y ) {4};
         
	\foreach \name / \y in {2}
        \path[yshift=0.5cm]
            node[leaf neuron] (HL2b-\name) at (3*\layersep,-11.5-\y ) {5};
     
    \foreach \name / \y in {3}
        \path[yshift=0.5cm]
            node[special leaf neuron] (HL2b-\name) at (3*\layersep,-11.5-\y ) {6};
 \foreach \name / \y in {4}
        \path[yshift=0.5cm]
            node[leaf neuron] (HL2b-\name) at (3*\layersep,-11.5-\y ) {8};
 \foreach \name / \y in {5}
        \path[yshift=0.5cm]
            node[leaf neuron] (HL2b-\name) at (3*\layersep,-11.5-\y ) {9};
 \foreach \name / \y in {6}
        \path[yshift=0.5cm]
            node[leaf neuron] (HL2b-\name) at (3*\layersep,-11.5-\y ) {10};


    \path (I-1) edge (HL1b-1)[black, thick];
    \path (I-1) edge (HL1b-3)[black, thick];
    \path (I-2) edge (HL1b-4)[black!40];
    \path (I-2) edge (HL1b-2)[black, thick];
    \path (I-2) edge (HL1b-5)[black!40];

\foreach \source in {1,2}
	 \path (HL1b-\source) edge (HL2b-1)[black!40];
\foreach \source in {1,2,3}
	 \path (HL1b-\source) edge (HL2b-3)[black!40];
\foreach \source in {1,4}
	 \path (HL1b-\source) edge (HL2b-4)[black!40];
\foreach \source in {1,4,5}
	 \path (HL1b-\source) edge (HL2b-5)[black!40];
\foreach \source in {1,4,5}
	 \path (HL1b-\source) edge (HL2b-6)[black!40];

   \foreach \source in {1,2,3}
	 \path (HL1b-\source) edge (HL2b-3)[blue, thick];

    \foreach \source in {1,...,\HLTwoN}
        \path (HL2b-\source) edge (O1)[black!40];

    \path (HL2b-3) edge (O1)[red!90, thick];

    \node[annot,above of=HL1-1, node distance=1.5cm] (hl1) { $\mathscr{H}$};
   \node[annot,left of=hl1, node distance=2.5cm] {Input layer};
    \node[annot,right of=hl1](hl2){ $\mathscr{L}$};
        \node[annot,right of=hl2, node distance=5cm](hl3){Output layer};





     \node[invisible neuron] (H1d) at (2*\layersep,-7 cm)  {$\hdots$};
     \node[invisible neuron] (H2d) at (3*\layersep,-7 cm)  {$\hdots$};

     \node[invisible neuron] (H1e) at (2*\layersep,-11 cm)  {$\hdots$};
     \node[invisible neuron] (H2e) at (3*\layersep,-11 cm)  {$\hdots$};

     \node[invisible neuron] (H1f) at (2*\layersep,-9 cm)  {$\hdots$};
     \node[invisible neuron] (H2f) at (3*\layersep,-9 cm)  {$\hdots$};


\end{tikzpicture}
\caption{Method 2: Joint training.}
\label{fig_method2}
\end{figure}
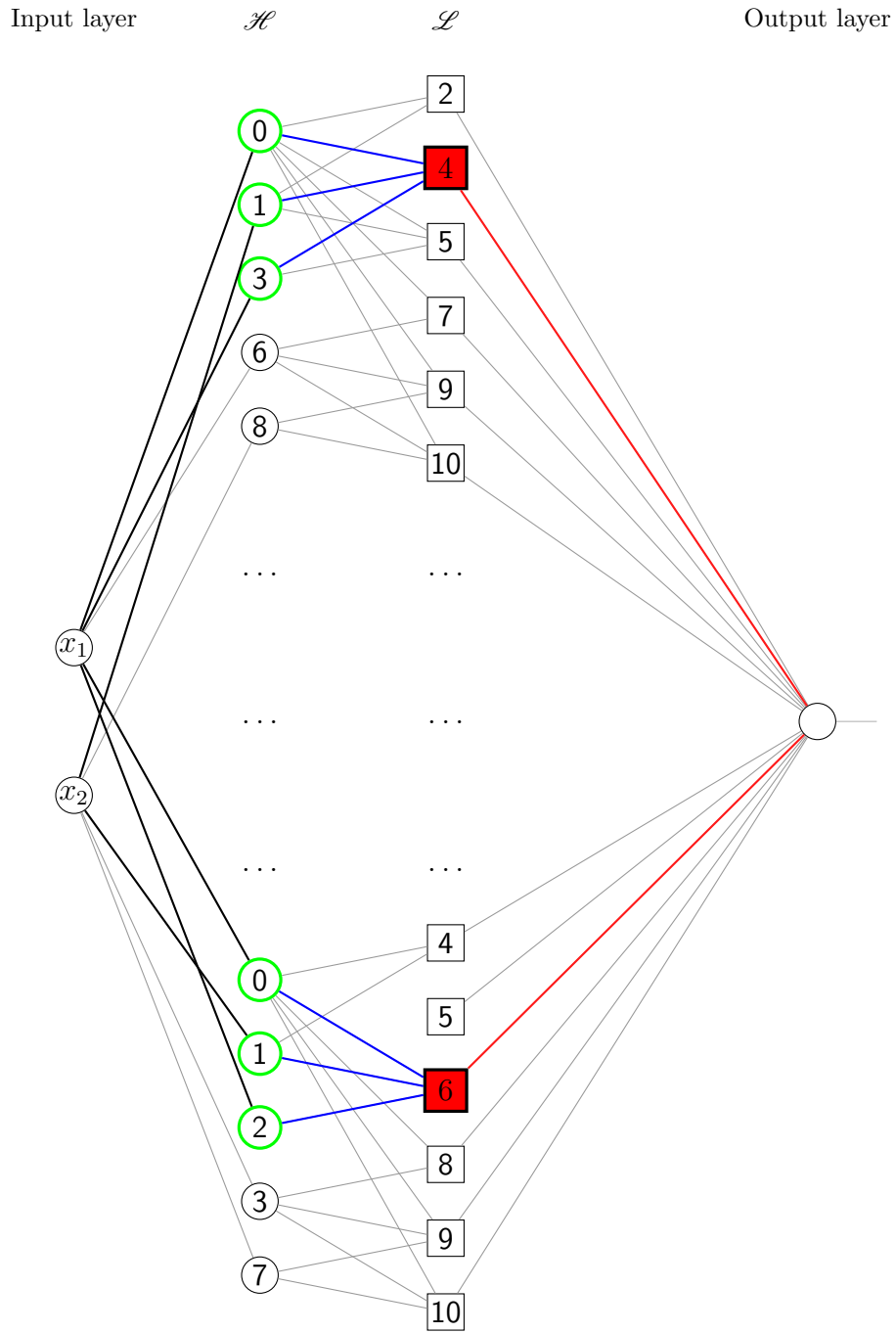

Although the hidden layers of the ``small'' networks are not connected across the sections belonging to each tree, this algorithm has two main differences with the previous one: 
\begin{enumerate}[$(i)$]
\item The output layer, which is shared by all ``small'' networks, computes a combination of {\it all} outputs of the second-layer neurons. 
\item The optimization is performed in one single run over the whole ``big'' network, and not network by network. Assuming that the trees are balanced, the first method performs, on average, $M$ different optimization programs in a space of $\mathscr O(K_n\log K_n)$ parameters, whereas the second one accomplishes only one minimization in a space of average dimension $\mathscr O(MK_n\log K_n)$.
\end{enumerate}
In the sequel, we let $s_{M,n}(\bx)$ be the regression function estimate corresponding to the second method. We note that this estimate still depends upon $\Theta_1, \hdots, \Theta_M$ and $\mathscr D_n$, but is {\it not} of the averaging form (\ref{average}). It will be more formally described in the next section.
\begin{rem}
The approach presented in this paper is based on the interpretation of trees as two-layer neural networks. This link does not immediately extend to more complex networks, especially those with more than two layers. One idea would be to artificially increase the depth of the networks by adding a collection of layers such that each newly added layer leaves the previous one intact. Nevertheless, such an architecture would probably be very difficult to optimize because of the numerous connections playing the same role.
\end{rem}
 \section{Some theory}
\subsection{Empirical risk minimization}
We now describe  in full detail the  construction of the two regression function estimates $r_{M,n}$ ({\bf independent training}) and $s_{M,n}$ ({\bf joint training}).

{\bf Method 1: Independent training.} Consider the $m$-th random tree in the ensemble and denote by $\mathscr G_1\equiv \mathscr G_1(\Theta_m,\mathscr D_n)$ the bipartite graph modeling the connections between the vector of inputs $\bx=(x^{(1)}, \hdots, x^{(d)})$ and the $K_n-1$ hidden neurons of the first layer. Similarly, let $\mathscr G_2\equiv \mathscr G_2(\Theta_m,\mathscr D_n)$ be the bipartite graph representing the connections between the first layer and the $K_n$ hidden neurons of the second layer.

Let $\mathbb M(\mathscr G_1)$ be the set of $d \times (K_n-1)$ matrices ${\bf W}_{1}=(a_{ij})$ such that  $a_{ij}=0$ if $(i,j)\notin \mathscr G_1$, and let $\mathbb M(\mathscr G_2)$ be the $(K_n-1)\times K_n$ matrices ${\bf W}_{2}=(b_{ij})$ such that  $b_{ij}=0$ if $(i,j)\notin \mathscr G_2$. The parameters that specify the first hidden units are encapsulated in a matrix ${\bf W}_{1}$ of $\mathbb M(\mathscr G_1)$ of weights over the edges of $\mathscr G_1$ and by a column vector of biases ${\bf b}_{1}$, of size $K_n-1$. Similarly, the parameters of the second hidden units are represented by a matrix ${\bf W}_{2}$ of $\mathbb M(\mathscr G_2)$ of weights over $\mathscr G_2$ and by a column vector ${\bf b}_{2}$ of offsets, of size $K_n$. Finally, we let the output weights and offset be ${\bf W}_{\!\textrm{out}}=(w_1, \hdots, w_{K_n})^{\top} \in \mathds R^{K_n}$ and ${b}_{\textrm{out}} \in \mathds R$, respectively ($\top$ denotes transposition and vectors are in column format).

Thus, the parameters that specify the $m$-th network are represented by a ``vector''
\begin{align*}
{\boldsymbol \lambda}=({\bf W}_{1},{\bf b}_{1},& {\bf W}_{2},{\bf b}_{2},{\bf W}_{\!\textrm{out}},{b}_{\textrm{out}}) \\
& \in \mathbb M(\mathscr G_1) \times \mathds R^{K_n-1}\times \mathbb M(\mathscr G_2) \times \mathds R^{K_n}\times \mathds R^{K_n}\times \mathds R.
\end{align*}
However, in order to obtain consistency we have to restrict the range of variation for these parameters. For a given matrix $M$, the notation $|M|$ means the matrix of absolute values of the entries of $M$. We assume that there exists a positive constant $C_1$ such that 
\begin{equation}
\label{cpoids}
\|{\bf W}_{2}\|_{\infty}+\|{\bf b}_{2}\|_{\infty}+\|{\bf W}_{\!\textrm{out}}\|_{1}+|{b}_{\textrm{out}}| \leq C_1 K_n,
\end{equation}
where $\|\cdot\|_{\infty}$ denotes the supremum norm of matrices and $\|\cdot\|_{1}$ is the $L_1$-norm of vectors. In other words, it is basically assumed that the weights and offsets (resp., the sum of absolute values of the weights and  the offset) absorbed by the computation units of the second layer (resp., the output layer) are at most of the magnitude of $K_n$.  We emphasize that this requirement is mild and that it leaves a lot of freedom for optimizing the parameters. We note in particular that it is satisfied by the original random tree estimates as soon as $Y$ is almost surely bounded, with the choice $C_1=(\frac{3}{2}+\|Y\|_{\infty})$ ($\|Y\|_{\infty}$ is the essential supremum of $Y$).

Therefore, letting
$$\Lambda(\Theta_m,\mathscr D_n)=\big\{{\boldsymbol \lambda}=({\bf W}_{1},{\bf b}_{1},{\bf W}_{2},{\bf b}_{2},{\bf W}_{\!\textrm{out}},{b}_{\textrm{out}}): (\ref{cpoids})\mbox{ is satisfied}\big\},$$
we see that the $m$-th neural network implements functions of the form
$$f_{{\boldsymbol \lambda}}(\bx)={\bf W}_{\!\textrm{out}}^{\top}\sigma_2\Big({\bf W}_{2}^{\top} \sigma_1({\bf W}_{1}^{\!\top}\bx+{\bf b}_{1})+{\bf b}_{2}\Big)+{b}_{\textrm{out}},\quad \bx \in \mathbb R^d,$$
where $\lambda \in \Lambda(\Theta_m,\mathscr D_n)$, and $\sigma_1$ and $\sigma_2$ are applied element-wise. Our aim is to adjust the parameter ${\boldsymbol \lambda}$ using the data $\mathscr D_n$ such that the function realized by the obtained network is a good estimate of $r$. Let
 $$\mathscr F(\Theta_m,\mathscr D_n)=\big\{f_{{\boldsymbol \lambda}}: {\boldsymbol \lambda} \in \Lambda(\Theta_m,\mathscr D_n)\big\}.$$
For each $m\in\{1, \hdots, M\}$, our algorithm constructs a regression function estimate $r(\cdot \,;\Theta_m,\mathscr D_n)$ by minimizing the empirical error 
$$J_n(f)=\frac{1}{n}\sum_{i=1}^n\big|Y_i-f(\bX_i)\big|^2$$
over functions $f$ in $\mathscr F(\Theta_m,\mathscr D_n)$, that is
$$J_n\big(r(\cdot \,;\Theta_m,\mathscr D_n)\big)\leq J_n(f)\quad \mbox{for all } f\in \mathscr F(\Theta_m,\mathscr D_n).$$
\begin{rem}
Here we assumed the existence of a minimum, though not necessarily its uniqueness. In cases where a minimum does not exist, the same analysis can be carried out with functions whose error is arbitrarily close to the infimum, but for the sake of simplicity we stay with the assumption of existence throughout the paper. Note also that we do not investigate the properties of the gradient descent algorithm used in Section $5$, and assume instead that the global minimum (if it exists) can be computed. 
\end{rem}
By repeating this minimization process for each $m\in \{1, \hdots, M\}$, we obtain a collection of (randomized) estimates $r(\cdot \,;\Theta_1,\mathscr D_n), \hdots, r(\cdot \,;\Theta_M,\mathscr D_n)$, which are aggregated to form the estimate 
\begin{equation*}
r_{M,n}(\bx)=\frac{1}{M}\sum_{m=1}^M r(\bx \,; \Theta_m,\mathscr D_n).
\end{equation*}
The estimate $r_{M,n}$ is but a generalization of the random forest estimate $t_{M,n}$ to the neural network framework, with an additional relaxation of crisp to fuzzy tree node membership due to the hyperbolic tangent activation functions: samples not merely fall into one direction per split and one final leaf but simultaneously into several tree branches and leaves.

{\bf Method 2: Joint training.} The notation needed to describe the second approach is a bit burdensome, but the ideas are simple. Following the above, we denote by $\mathscr G_{1,1}, \hdots, \mathscr G_{1,M}$ and $\mathscr G_{2,1}, \hdots, \mathscr G_{2,M}$ the bipartite graphs associated with the $M$ ``small'' original random trees. We also let ${\bf W}_{1,1},\hdots, {\bf W}_{1,M}$ and ${\bf b}_{1,1},\hdots, {\bf b}_{1,M}$ be the respective weight matrices and offset vectors of the first hidden layers of the $M$ ``small'' networks, with ${\bf W}_{1,m} \in \mathbb M(\mathscr G_{1,m})$ and ${\bf b}_{1,m} \in \mathbb R^{K_n-1}$, $1\leq m \leq M$. Similarly, we denote by ${\bf W}_{2,1},\hdots, {\bf W}_{2,M}$ and ${\bf b}_{2,1},\hdots, {\bf b}_{2,M}$ the respective weight matrices and offset vectors of the second layer, with ${\bf W}_{2,m} \in \mathbb M(\mathscr G_{2,m})$ and ${\bf b}_{2,m} \in \mathbb R^{K_n}$, $1\leq m\leq M$.

Next, we form the concatenated matrices $[{\bf W}_{1}]$, $[{\bf b}_{1}]$, $[{\bf W}_{2}]$, and $[{\bf b}_{2}]$, defined by
\begin{equation*}
[{\bf W}_{1}]=
\begin{pmatrix}
   {\bf W}_{1,1} & \cdots &{\bf W}_{1,M}
\end{pmatrix},
\quad
[{\bf b}_{1}]=
\begin{pmatrix}
   {\bf b}_{1,1} \\
   \vdots\\
   {\bf b}_{1,M}
\end{pmatrix},
\end{equation*}
and
\begin{equation*}
[{\bf W}_{2}]=
\begin{pmatrix}
  {\bf W}_{2,1} & {\bf 0}&\cdots &{\bf 0}& {\bf 0}\\
{\bf 0}  &{\bf W}_{2,2}& \cdots &{\bf 0} &{\bf 0}\\
  \vdots  &\vdots&\ddots&\vdots& \vdots\\
   {\bf 0}  &{\bf 0} &\hdots&{\bf W}_{2,M-1}&{\bf 0}\\
  {\bf 0} & {\bf 0} &\cdots&{\bf 0} &{\bf W}_{2,M}
\end{pmatrix},
\quad
[{\bf b}_{2}]=
\begin{pmatrix}
   {\bf b}_{2,1} \\
   \vdots\\
   {\bf b}_{2,M}
\end{pmatrix}.
\end{equation*}
Notice that $[{\bf W}_{1}]$, $[{\bf b}_{1}]$, $[{\bf W}_{2}]$, and $[{\bf b}_{2}]$ are of size $d\times M(K_n-1)$, $M(K_n-1)\times 1$, $M(K_n-1)\times MK_n$, and $MK_n\times 1$, respectively. Let us finally denote by ${\bf W}_{\!\textrm{out}}\in \mathds R^{MK_n}$ and ${b}_{\textrm{out}} \in \mathbb R$ the output weights and offset of the concatenated network.  All in all, the parameters of the network are represented by a ``vector''
$$[{\boldsymbol \lambda}]=\big([{\bf W}_{1}], [{\bf b}_{1}], [{\bf W}_{2}], [{\bf b}_{2}],{\bf W}_{\!\textrm{out}},{b}_{\textrm{out}}\big),$$
where $[{\bf W}_{1}]$, $[{\bf b}_{1}]$, $[{\bf W}_{2}]$, and $[{\bf b}_{2}]$ are defined above. As in the first method, we restrict the range of variation of these parameters and assume that there exists a positive constant $C_2$ such that 
\begin{equation}
\label{cpoids2}
\big\|[{\bf W}_{2}]\big\|_{\infty}+\big\|[{\bf b}_{2}]\big\|_{\infty}+\|{\bf W}_{\!\textrm{out}}\|_{1}+|{b}_{\textrm{out}}|  \leq C_2 K_n.
\end{equation}
Therefore, letting
\begin{align*}
    \Lambda(\Theta_1, & \hdots,  \Theta_M,\mathscr D_n)\\
& =\big\{[{\boldsymbol \lambda}]=\big([{\bf W}_{1}],[{\bf b}_{1}],[{\bf W}_{2}],[{\bf b}_{2}],{\bf W}_{\!\textrm{out}},{b}_{\textrm{out}}\big): (\ref{cpoids2})\mbox{ is satisfied}\big\},
\end{align*}
the ``big'' network implements functions of the form
$$f_{[{\boldsymbol \lambda}]}(\bx)={\bf W}_{\!\textrm{out}}^{\top}\sigma_2\Big([{\bf W}_{2}]^{\top} \sigma_1\big([{\bf W}_{1}]^{\!\top}\bx+[{\bf b}_{1}]\big)+[{\bf b}_{2}]\Big)+{b}_{\textrm{out}},\quad \bx\in \mathbb R^d,$$
where $[\boldsymbol \lambda]\in \Lambda(\Theta_1, \hdots, \Theta_M,\mathscr D_n)$, and $\sigma_1$ and $\sigma_2$ are applied element-wise. Next, let
 $$\mathscr F(\Theta_1, \hdots, \Theta_M,\mathscr D_n)=\big\{f_{[{\boldsymbol \lambda}]}: {[\boldsymbol \lambda]} \in \Lambda(\Theta_1, \hdots, \Theta_M,\mathscr D_n)\big\}.$$
Then the final estimate $s_{M,n}$ is obtained by minimizing the empirical error 
$$J_n(f)=\frac{1}{n}\sum_{i=1}^n\big|Y_i-f(\bX_i)\big|^2$$
over functions $f$ in $\mathscr F(\Theta_1, \hdots, \Theta_M,\mathscr D_n)$, that is
$$J_n(s_{M,n})\leq J_n(f)\quad \mbox{for all } f\in \mathscr F(\Theta_1, \hdots, \Theta_M,\mathscr D_n).$$
\subsection{Consistency}
In order to analyze the consistency properties of the regression function estimates $r_{M,n}$ and $s_{M,n}$, we first need to consider some specific class $\mathscr F$ of functions over $[0,1]^d$. It is defined as follows.

For a hyperrectangle $A = [a_1, b_1] \times \cdots \times [a_d, b_d] \subseteq [0,1]^d$, we let $A^{\backslash j} = \prod_{i \neq j} [a_i, b_i]$ and $\textrm{d} \bx^{\backslash j} = \textrm{d}x_1 \hdots \textrm{d}x_{j-1} \textrm{d}x_{j+1} \hdots \textrm{d}x_d$. Assume we are given a measurable function $f:[0,1]^d\to \mathbb R$ together with $A = [a_1, b_1] \times \cdots \times [a_d, b_d] \subseteq [0,1]^d$, and consider the following two statements:
\begin{itemize}
\item[$(i)$] For any $j \in \{1, \hdots, d\}$, the function 
\begin{align*}
x_j \mapsto \int_{A^{\backslash j}} f(\bx) \textrm{d} \bx^{\backslash j}
\end{align*}
is constant on $[a_j,b_j]$; 
\item[$(ii)$] The function $f$ is constant on $A$. 
\end{itemize}
\begin{defi}
\label{proper_reg_funct}
We let $\mathscr{F}$ be the class of continuous real functions on $[0,1]^d$ such that, for any $A = [a_1, b_1] \times \cdots \times [a_d, b_d] \subseteq [0,1]^d$, $(i)$ implies $(ii)$.
\end{defi}
Although the membership requirement for $\mathscr{F}$ seems at first glance a bit restrictive, it turns out that $\mathscr{F}$ is in fact a rich class of functions. For example, additive functions of the form 
\begin{align*}
f(\bx) = \sum_{j=1}^d f_j(x^{(j)}),
\end{align*}
where each $f_j$ is continuous, do belong to $\mathscr F$. This is also true for polynomial functions whose coefficients have the same sign. Also, product of continuous functions of the form
\begin{align*}
f(\bx) = \prod_{j=1}^d f_j(x^{(j)}),
\end{align*}
where, for all $j \in \{1, \hdots, d\}$, [$f_j >0$ or $f_j <0$], are included in $\mathscr F$.

Our main theorem states that the neural forest estimates $r_{M,n}$ and $s_{M,n}$ are consistent, provided the number $K_n$ of terminal nodes and the parameters $\gamma_1$ and $\gamma_2$ are properly regulated as functions of $n$.
\begin{theo}[{\bf Consistency of $r_{M,n}$ and $s_{M,n}$}]
\label{maintheo1}
Assume that $\bX$ is uniformly distributed in $[0,1]^d$,  $\|Y\|_{\infty}<\infty$, and $r \in \mathscr{F}$. Assume, in addition, that $K_n, \gamma_1, \gamma_2 \to\infty$ such that, as $n$ tends to infinity,  
\begin{align*}
    \frac{K_n^6\log (\gamma_2  K_n^5)}{n}\to 0, \quad K_n^{2}e^{-2 \gamma_2} \to 0, \quad \textrm{and} \quad \frac{K_n^4\gamma_2^2\log(\gamma_1)}{\gamma_1}\to 0.
\end{align*}
Then, as $n$ tends to infinity,  
$$\mathds E \big|r_{M,n}(\bX)-r(\bX)\big|^2 \to 0 \quad \mbox{and} \quad\mathds E \big|s_{M,n}(\bX)-r(\bX)\big|^2 \to 0.$$
\end{theo}
It is interesting to note that Theorem \ref{maintheo1} still holds when the individual trees are subsampled and fully grown (that is, when $K_n = a_n$, i.e., one single observation in each leaf) as soon as the assumptions are satisfied with $a_n$ instead of $K_n$. Put differently, we require that the trees of the forest are either pruned (restriction on $K_n$) or subsampled (restriction on $a_n$). If not, 
the assumptions of Theorem \ref{maintheo1} are violated. 
So, to obtain a consistent prediction, it is therefore mandatory to keep the depth of the tree or the size of subsamples under control. 
We also note that Theorem \ref{maintheo1} can be adapted to deal with networks based on bootstrapped and fully grown trees. In this case, $n$ observations are chosen in $\mathscr D_n$ {\it with} replacement prior to each tree construction, and only one distinct example is left in the leaves. In this setting, care must be taken in the analysis to consider only the trees that use at least $K_n$ distinct data points---we leave the adaptation as a small exercise.

Let us finally point out that the proof of Theorem \ref{maintheo1} relies on the consistency of the individual ``small'' networks. Therefore, the proposed analysis does not really highlight the benefits of aggregating individual trees in terms of finite-sample analysis or asymptotic behavior. Undoubtedly, there is  room for further improvement.

\subsection{Proof of Theorem \ref{maintheo1}}
\paragraph{Consistency of $r_{M,n}$.}
Denote by $\mu$ the distribution of $\bX$. The consistency proof of $r_{M,n}$ starts with the observation that
\begin{align*}
 \mathds E \big|r_{M,n}(\bX)-r(\bX)\big|^2 &=  \mathds E \bigg| \frac{1}{M}\sum_{m=1}^M r(\bX \,;\Theta_m,\mathscr D_n)-r(\bX)\bigg|^2\\
 &\leq \frac{1}{M}\sum_{m=1}^M \mathds E \big|  r(\bX \,;\Theta_m,\mathscr D_n)-r(\bX)\big|^2\\
 & \quad \mbox{(by Jensen's inequality)}\\
 & = \mathds E \big|  r(\bX \,;\Theta,\mathscr D_n)-r(\bX)\big|^2.
\end{align*}
Therefore, we only need to show that, under the conditions of the theorem, $\mathds E |  r(\bX \,;\Theta,\mathscr D_n)-r(\bX)|^2\to 0$,  i.e., that a single random network estimate is consistent (note that the expectation is taken with respect to $\bX$, $\Theta$, and $\mathscr D_n$). 

We have (\citealp[see, e.g.,][]{LuZe95}, or \citealp[][Lemma 10.1]{GyKoKr02})
\begin{align}
&\mathds E \big|  r(\bX \,;\Theta,\mathscr D_n)-r(\bX)\big|^2 \nonumber \\
& \quad \leq 2 \mathds E \sup_{f \in \mathscr F(\Theta,\mathscr D_n)}\bigg|\frac{1}{n}\sum_{i=1}^n\big|Y_i-f(\bX_i)\big|^2-\mathds E\big|Y-f(\bX)\big|^2\bigg| \nonumber\\
&\qquad + \mathds E \inf_{f \in \mathscr F(\Theta,\mathscr D_n)} \int_{[0,1]^d}\big|f(\bx)-r(\bx)\big|^2\mu(\mbox{d}\bx). \label{E+A}
\end{align}
The first term---the estimation error---is handled in Proposition \ref{prop-ee} below by using nonasymptotic uniform deviation inequalities and covering numbers corresponding to $\mathscr F(\Theta,\mathscr D_n)$ (proofs are in Section 6).
\begin{pro}
\label{prop-ee}
Assume that $K_n, \gamma_2 \to \infty$ such that $K_n^6\log (\gamma_2 K_n^5)/n\to 0$. Then
$$\mathds E \sup_{f \in \mathscr F(\Theta,\mathscr D_n)}\bigg|\frac{1}{n}\sum_{i=1}^n\big|Y_i-f(\bX_i)\big|^2-\mathds E\big|Y-f(\bX)\big|^2\bigg| \to 0\quad \mbox{as } n \to \infty.$$
\end{pro}
To deal with the second term of the right-hand side of (\ref{E+A})---the approximation error---, we consider a piecewise constant function (pseudo-estimate) similar to the original CART-tree $t_n(.\,;\Theta,\mathscr D_n)$, with only one difference: the function computes the true conditional expectation $\mathbb E[Y|\bX \in L_{k'}]$ in each leaf $L_{k'}$, not the empirical one $\bar Y_{k'}$. Put differently, we take $({\bf W}_{\!\textrm{out}}^{\star})_{k'}= \mathbb E[Y|\bX \in L_{k'}]/2$ and ${b}_{\textrm{out}}^{\star}=\sum_{k'=1}^{K_n} \mathbb E[Y|\bX \in L_{k'}]/2$ in (\ref{NFI}). This tree-type pseudo-estimate has the form
$$t_{\boldsymbol \lambda^{\star}}(\bx)={\bf W}_{\!\textrm{out}}^{\star \top}\tau\Big({\bf W}_{2}^{\star\top} \tau({\bf W}_{1}^{\star\top}\bx+{\bf b}_{1}^{\star})+{\bf b}_{2}^{\star}\Big)+b_{\textrm{out}}^{\star},\quad \bx \in \mathbb R^d,$$
(recall that $\tau(u)=2\mathds1_{u\geq 0}-1$), for some ${\boldsymbol \lambda^{\star}}=({\bf W}_{1}^{\star},{\bf b}_{1}^{\star},{\bf W}_{2}^{\star},{\bf b}_{2}^{\star},{\bf W}_{\!\textrm{out}}^{\star},b_{\textrm{out}}^{\star})$. We have
\begin{align*}
&\inf_{f \in \mathscr F(\Theta,\mathscr D_n)} \int_{[0,1]^d}\big|f(\bx)-r(\bx)\big|^2\mu(\mbox{d}\bx)\\
&\quad = \inf_{f \in \mathscr F(\Theta,\mathscr D_n)} \int_{[0,1]^d}\big|f(\bx)-r(\bx)\big|^2\mu(\mbox{d}\bx)- 2\int_{[0,1]^d}\big|t_{\boldsymbol \lambda^{\star}}(\bx)-r(\bx)\big|^2\mu(\mbox{d}\bx)\\
& \qquad + 2\int_{[0,1]^d}\big|t_{\boldsymbol \lambda^{\star}}(\bx)-r(\bx)\big|^2\mu(\mbox{d}\bx).
\end{align*}
Therefore,
\begin{align*}
&\inf_{f \in \mathscr F(\Theta,\mathscr D_n)} \int_{[0,1]^d}\big|f(\bx)-r(\bx)\big|^2\mu(\mbox{d}\bx)\\
&\quad \leq \int_{[0,1]^d} \big|f_{\boldsymbol \lambda^{\star}}(\bx)-r(\bx)\big|^2\mu(\mbox{d}\bx)- 2\int_{[0,1]^d}\big|t_{\boldsymbol \lambda^{\star}}(\bx)-r(\bx)\big|^2\mu(\mbox{d}\bx)\\
& \qquad + 2\int_{[0,1]^d}\big|t_{\boldsymbol \lambda^{\star}}(\bx)-r(\bx)\big|^2\mu(\mbox{d}\bx)\\
& \quad  \leq 2\int_{[0,1]^d} \big|f_{\boldsymbol \lambda^{\star}}(\bx)-t_{\boldsymbol \lambda^{\star}}(\bx)\big|^2\mu(\mbox{d}\bx)+2\int_{[0,1]^d}\big|t_{\boldsymbol \lambda^{\star}}(\bx)-r(\bx)\big|^2\mu(\mbox{d}\bx).
\end{align*}
We prove in Proposition \ref{prop-lambda} that the expectation of the first of the two terms above tends to zero under appropriate conditions on $\gamma_1$ and $\gamma_2$. The second term is less standard and requires a careful analysis of the asymptotic geometric behavior of the cells of the tree pseudo-estimate $t_{\boldsymbol \lambda^{\star}}$. This is the topic of Proposition \ref{prop-diameter}. Taken together, inequality (\ref{E+A}) and Proposition \ref{prop-ee}-\ref{prop-diameter} prove the result.
\begin{pro}
\label{prop-lambda}
Assume that $\bX$ is uniformly distributed in $[0,1]^d$ and $\|Y\|_{\infty}<\infty$. Assume, in addition, that $K_n, \gamma_1, \gamma_2 \to\infty$ such that 
\begin{align*}
    K_n^{2}e^{-2 \gamma_2} \to 0 \quad \textrm{and} \quad  K_n^4\gamma_2^2\log(\gamma_1)/\gamma_1\to 0.
\end{align*}
Then
$$\mathbb E\int_{[0,1]^d} \big|f_{\boldsymbol \lambda^{\star}}(\bx)-t_{\boldsymbol \lambda^{\star}}(\bx)\big|^2\mu(\emph{d}\bx) \to 0\quad \mbox{as }n \to \infty.$$
\end{pro}
\begin{pro}
\label{prop-diameter}
Assume that $\bX$ is uniformly distributed in $[0,1]^d$ and $\|Y\|_{\infty}<\infty$. If $r \in \mathscr{F}$, then
$$\mathbb E\int_{[0,1]^d} \big|t_{\boldsymbol \lambda^{\star}}(\bx)-r(\bx)\big|^2\mu(\emph{d}\bx) \to 0\quad \mbox{as }n \to \infty.$$
\end{pro}
\paragraph{Consistency of $s_{M,n}$.}
The consistency of $s_{M,n}$ is a consequence of the first statement of Theorem \ref{maintheo1}. Denote by $\widehat{\bf W}_{1,1}, \hdots, \widehat{\bf W}_{1,M}$, $\widehat {\bf b}_{1,1}, \hdots, \widehat {\bf b}_{1,M}$, $\widehat{\bf W}_{2,1}, \hdots, \widehat{\bf W}_{2,M}$, $\widehat {\bf b}_{2,1}, \hdots, \widehat {\bf b}_{2,M}$, $\widehat {\bf W}_{\!\textrm{out}, 1}, \hdots, \widehat {\bf W}_{\!\textrm{out},M}$, and $\widehat b_{\textrm{out},1}, \hdots, \widehat b_{\textrm{out},M}$ the set of parameters (weights and offsets) output by the minimization programs performed at each tree-type network of {\bf Method 1}. It is then easy to see that the ``big'' network fitted with the parameters 
$$[\widehat {\boldsymbol \lambda}]=\big ([\widehat {\bf W}_{1}], [\widehat {\bf b}_{1}], [\widehat {\bf W}_{2}], [\widehat {\bf b}_{2}],[\widehat {\bf W}_{\!\textrm{out}}],\widehat {b}_{\textrm{out}}\big),$$
where
\begin{equation*}
[\widehat {\bf W}_{\!\textrm{out}}]=\frac{1}{M}
\begin{pmatrix}
 \widehat {\bf W}_{\!\textrm{out},1} \\
  \vdots  \\
 \widehat  {\bf W}_{\!\textrm{out},M}
\end{pmatrix}
\end{equation*}
and $\widehat {b}_{\textrm{out}} =\frac{1}{M}\sum_{m=1}^M \widehat {b}_{\textrm{out},m}$, exactly computes the function $r_{M,n}$. This implies 
\begin{align*}
\frac{1}{n}\sum_{i=1}^n \big|Y_i-s_{M,n}(\bX_i)\big|^2 & \leq \frac{1}{n} \sum_{i=1}^n \big|Y_i-r_{M,n}(\bX_i)\big|^2 \nonumber \\
& \leq \frac{1}{M} \sum_{m=1}^M \Big( \frac{1}{n} \sum_{i=1}^n \big|Y_i-r(\bX_i\,; \Theta_m, \mathscr{D}_n)\big|^2 \Big),
\end{align*}
by Jensen's inequality.  Denote by $\mathds{E}_{\bX, Y}$ the expectation with respect to $\bX$ and $Y$ only (that is, all the other random variables are kept fixed). Thus, simple calculations show that 
\begin{align}
    &  \mathds{E} \big|s_{M,n}(\bX) - r(\bX)\big|^2 \nonumber \\
    & \leq \mathds{E} \left| \mathds{E} \big|Y - s_{M,n}(\bX)\big|^2 - \frac{1}{n}\sum_{i=1}^n \big|Y_i-s_{M,n}(\bX_i)\big|^2 \right| \nonumber \\
    & \quad + \mathds{E} \left[ \frac{1}{M} \sum_{m=1}^M \Big| \frac{1}{n} \sum_{i=1}^n \big|Y_i-r(\bX_i\,; \Theta_m, \mathscr{D}_n)\big|^2  - \mathds{E}_{\bX,Y} \big|Y - r(\bX\,; \Theta_m, \mathscr{D}_n)\big|^2\Big| \right] \nonumber\\
    & \quad + \mathds{E} \left[ \frac{1}{M} \sum_{m=1}^M \Big(\mathds{E}_{\bX, Y} \big|Y - r(\bX\,; \Theta_m, \mathscr{D}_n)\big|^2 - \mathds{E} \big|Y - r(\bX)\big|^2\Big)\right]. \label{proof_big_net_1}
\end{align}
Regarding the third term on the right-hand side and under the assumptions of Theorem \ref{maintheo1}, we have 
\begin{align}
    & \mathds{E} \Bigg[ \frac{1}{M} \sum_{m=1}^M \Big(\mathds{E}_{\bX, Y} \big|Y - r(\bX\,; \Theta_m, \mathscr{D}_n)\big|^2 - \mathds{E} \big|Y - r(\bX)\big|^2\Big) \Bigg] \nonumber \\
    & =  \mathds{E} \Bigg[ \frac{1}{M} \sum_{m=1}^M \mathds{E}_{\bX, Y} \big|r(\bX) - r(\bX\,; \Theta_m, \mathscr{D}_n)\big|^2 \Bigg] \nonumber \\
    & =  \mathds{E} \big|r(\bX) - r(\bX\,; \Theta_m, \mathscr{D}_n)\big|^2 \nonumber \\
    & \quad \to 0 \quad \textrm{as} ~ n \to \infty, \label{proof_th2_0}
\end{align}
since $r(\bX\,; \Theta_m, \mathscr{D}_n)$ is consistent by the first statement of Theorem \ref{maintheo1}. Regarding the second term in (\ref{proof_big_net_1}), according to the proof of Proposition \ref{prop-ee},
\begin{align}
& \mathds{E} \Bigg[ \frac{1}{M} \sum_{m=1}^M  \bigg| \frac{1}{n}\sum_{i=1}^n \big|Y_i-r(\bX_i\,; \Theta_m, \mathscr{D}_n)\big|^2 - \mathds{E}_{\bX, Y} \big|Y - r(\bX\,; \Theta_m, \mathscr{D}_n)\big|^2 \bigg| \Bigg] \nonumber \\
& \leq  \mathds E \sup_{f \in \mathscr F_n}\bigg|\frac{1}{n}\sum_{i=1}^n\big|Y_i-f(\bX_i)\big|^2-\mathds E\big|Y-f(\bX)\big|^2\bigg| \nonumber \\
& \quad \to 0\quad \mbox{as } n \to \infty,
\label{proof_th2_1}
\end{align}
where the set $\mathscr F_n$ contains all neural networks, constrained by (\ref{cpoids}).
A similar result can be obtained for the ``big'' network $s_{M,n}$, just by replacing $K_n$ (an upper bound over the number of neurons in the two layers of the ``small'' networks) by $MK_n$. Thus, since $M$ is fixed, still under the assumptions of Theorem \ref{maintheo1}, we may write
\begin{align}
& \mathds{E} \bigg| \frac{1}{n}\sum_{i=1}^n \big|Y_i-s_{M,n}(\bX_i)\big|^2 - \mathds{E} \big|Y - s_{M,n}(\bX)\big|^2 \bigg| \nonumber \\
& \leq  \mathds E \sup_{f \in \mathscr F(\Theta_1, \hdots, \Theta_M,\mathscr D_n)}\bigg|\frac{1}{n}\sum_{i=1}^n\big|Y_i-f(\bX_i)\big|^2-\mathds E\big|Y-f(\bX)\big|^2\bigg| \nonumber \\
& \quad \to 0\quad \mbox{as } n \to \infty.
\label{proof_th2_2}
\end{align}
Therefore, taking expectation on both sides of (\ref{proof_big_net_1}), and assembling inequalities (\ref{proof_th2_0}), (\ref{proof_th2_1}), and (\ref{proof_th2_2}), we have
\begin{align*}
    \mathds{E} \big|s_{M,n}(\bX) - r(\bX)\big|^2 \to 0 \quad \textrm{as} ~ n \to \infty, 
\end{align*}
which concludes the proof. 
\section{Experiments}
In this section we validate the Neural Random Forest (NRF) models experimentally (Method 1 and Method 2). We compare them with standard Random Forests (RF), Neural Networks (NN) with one, two, and three hidden layers, as well as Bayesian Additive Regression Trees (\citealp[BART,][]{chipman2010bart}).
\subsection{Training procedure}
Overall, the training procedure for the NRF models goes as follows. A random forest is first learned using the {\it scikit-learn} implementation \citep{Pe12} for random forests. Based on this, the set of all split directions and split positions are extracted and used to build the neural network initialization parameters. The NRF models are then trained using the {\it TensorFlow} framework \citep{Ma15}.
\paragraph{Network optimization.}
The optimization objective when learning either of the network models (NRF or standard NN) is to minimize the mean squared error (MSE) on some training set. In neural network training, this is typically achieved by employing an iterative gradient-based optimization algorithm. The algorithm iterates over the training set, generates predictions, and then propagates the gradient of the resulting error signal with respect to all individual network parameters back through the network. The network parameters are updated such that this error is decreased, and slowly the model learns to generate the correct predictions. In practice, we found that {\it Adam} \citep{Ki14}, which also includes an adaptive momentum term, performed well as optimization algorithm, though any other iterative gradient-based optimization algorithm could be used instead. The results reported here were obtained with {\it Adam} using minibatches of size 32, default hyperparameter values ($\beta_1=0.9$, $\beta_2=0.999$, $\varepsilon=1e-08$), and initial learning rate $0.001$ for which a stable optimization behavior could be observed.
In preliminary experiments, minibatch size did not have a big impact on performance and was thus not tuned. At the beginning of every new epoch, the training set was shuffled and minibatches assigned anew so as to avoid overfitting to a specific order or choice of particular minibatches. 

Each neural network (including the NRF networks) was trained for 100 epochs. During training, both training loss and validation loss were monitored every time an epoch was completed. The final parameters chosen are the ones that gave minimum validation error across all 100 epochs. 
\paragraph{Neural random forests.} The individual networks in Method 1 are trained just like the larger network in Method 2 for 100 epochs, and the best parameters with minimum loss on the validation set during training were picked. Computed sequentially, training Method 1 takes longer since each ``small'' model has to be fitted individually.

We generally found that using a lower value for $\gamma_2$ than for $\gamma_1$ (the initial contrast parameters of the activation functions in the second and first hidden layer, respectively) is helpful. This is because with a relatively small contrast $\gamma_2$, the transition in the activation function from $-1$ to $+1$ is smoother and a stronger gradient signal reaches the first hidden layer in backpropagation training. 
Concretely, for our experiments we used $\gamma_1=100$ and $\gamma_2=1$.  

In some rare cases, the NRF just overfitted during optimization. This is, even though the training error went down, validation error only increased, i.e., the model's ability to generalize actually suffered from network optimization. Wherever NRF validation loss was actually worse than the RF validation loss during {\it all} epochs, we kept the original RF model predictions. As a consequence, we can expect the NRF predictions to be at least as good as the RF predictions, since cases where further optimization is likely to have lead to overfitting are directly filtered out.
\paragraph{Sparse vs. fully connected optimization.} All NRF networks described in the previous sections have sparse connectivity architecture between the layers due to the nature of the translated tree structures. However, besides training sparse NRF networks, a natural modification is to not limit network training to optimizing a small set of weights, but instead to relax the sparsity constraint and train a fully connected feed-forward network. With initial connection weights of 0 between neurons where no connections were described in the sparse framework, this model still gives the same predictions as the initial sparsely connected network. However, it can adapt a larger set of weights for optimizing the objective. During the experiments, we will refer to this relaxed NRF version as {\it fully connected}, in contrast to the {\it sparse} setting, where fewer weights are subject to optimization.
 Both algorithms fight overfitting in their own way: the sparse network takes advantage of the forest structure during the whole optimization process, whereas the fully connected network uses the smart initialization provided by the forest structure as a warm start. Indeed, the tendency of fully connected networks to overfit on small data sets can potentially be reduced with the inductive bias derived from the trees.
 	
We also stress that the fully connected approach differs from the sparse one only through the relaxation of the sparsity constraint during the optimization process. Apart from that, both approaches are declined with the same architectures: Method 1 (disconnected networks) or Method 2 (a ``big'' network). In fact, the fully connected approach should rather be understood as an algorithmic trick to initialize the network weights, as part of a traditional optimization procedure. In practice, without a fully differentiable implementation of sparse models, the fully connected models can be faster to optimize than their sparse counterpart, especially when training on a GPU. We used a dense matrix multiplication, forcing the entries of non-existing connections to 0. We acknowledge that this is not the most elegant solution and a fully differentiable implementation of sparse matrix multiplication could accelerate training here, though we did not pursue this way further.
 
\subsection{Benchmark comparison experiments}
We now compare the NRF models with standard RF, standard NN, and BART on regression data sets from UCI Machine Learning Repository \citep{Li13}. Concretely, we use the {\it Auto MPG}, {\it Housing}, {\it Communities and Crime} \citep{Re02}, {\it Forest Fires} \citep{Co07}, {\it Breast Cancer Wisconsin (Prognostic)}, {\it Concrete Compressive Strength} \citep{Ye98}, and {\it Protein Tertiary Structure} data sets as testing ground for the models.

To showcase the abilities of the NRF, we picked diverse, but mostly small regression data sets with few samples. These are the sets where RF often perform better than NN, since the latter typically require plentiful training data to perform reasonably well. It is on these small data sets where the benefits of neural optimization can usually not be exploited, but with the specific inductive bias of the NRF it becomes possible. We also study NRF performance on a larger data set, {\it Protein Tertiary Structure},  which contains $45000$ observations. 

Random forests are trained with 30 trees and maximum depth restriction of 6. The neural networks with one, two, or three hidden layers (NN1, NN2, NN3) trained for comparison are fully connected and have the same number of neurons in each layer as the NRF networks of Method 2 (for the third layer of NN3, the number of neurons is the same as in the second layer of NRF2). The initial parameters of standard NN were drawn randomly from a standard Gaussian distribution, and the above training procedure was followed. 

For comparison we also evaluate Bayesian Additive Regression Trees (\citealp[BART,][]{chipman2010bart}).
BART is trained also with 30 trees, 1000 MCMC steps (after burn-in of 100 steps), and otherwise default hyperparameters from the {\it BayesTree} R implementation. 
\paragraph{Data preparation.} We shuffled the observations in each data set randomly and split it into training, validation, and test part in a ratio of 50/25/25. Each experiment is repeated 10 times with different randomly assigned training, validation, and test set. 

\begin{table*}[!h]
\footnotesize
    \centering
    \begin{tabular}{@{\extracolsep{4pt}}|l| l | l |}
    \hline
        \hline
        {\bf Data set} & {\bf Number of samples} & {\bf Number of features}  \\ 
        \hline
        {\bf Auto MPG}    & 398 & 7  \\
        {\bf Housing}     & 506 & 13 \\ 
        {\bf Communities and Crime}   & 1994 & 101  \\ 
        {\bf Forest Fires}        & 517 & 10  \\ 
        {\bf Wisconsin}           & 194 & 32 \\
        {\bf Concrete}           & 1030 & 8\\
        {\bf Protein}           & 45730 & 9\\
        \hline
        \hline
    \end{tabular}
    \caption{Data set characteristics: number of samples and number of features, after removing observations with missing information or nonnumerical input features.}
    \label{dataset_stats}
\end{table*}
Nonnumerical features and samples with missing entries were systematically removed. Table \ref{dataset_stats} summarizes the characteristics of the resulting data sets. In preliminary experiments, data normalisation or rescaling to the unit cube did not  have a big impact on the models, so we kept the original values given in each of the sets.

\paragraph{Results.} Table \ref{results} summarizes the results in terms of RMSE (root mean squared error) for the different models. The first important comment is that all NRF models improve over the original RF, consistently across all data sets. So the NRF are indeed more competitive than the original random forests which they were derived from. These consistent improvements over the original RF can be observed both for sparse and fully connected NRF networks.

\begin{table*}[!h]
\footnotesize
    \centering
    \begin{tabular}[width=\textwidth]{@{\extracolsep{1pt}}|r|r|r|r|r|@{}}
    \hline
    \hline
        {\bf Data set} 		& {\bf NN1} 	& {\bf NN2} 		& {\bf NN3}	& {\bf RF} 			 \\ 
        \hline

        {\bf Auto MPG}  	& 4.56 (0.83) 	& 3.95 (0.39) 		& 5.02 (0.72)	& 3.44 (0.38) 		\\
        {\bf Housing}         	& 9.06 (0.85) 	& 7.81 (0.71) 		& 12.52 (1.08)	& 4.78 (0.88) 		\\
        {\bf Crime}             	& 5.39 (0.42) 	& 6.78 (0.47) 		& 6.64 (0.31)	& 0.17 (0.01)  		\\
        {\bf Forest Fires}       	& 96.7 (0.20) 	& {\it54.87 (34.33)}  	& 97.9 (0.90)	& 95.47 (43.52) 	\\
        {\bf Wisconsin}         	& 36.91 (0.88) 	& {\it34.71 (2.36)} 	& 38.43 (2.88)	& 45.63 (3.53) 		\\
        {\bf Concrete}            & 10.18 (0.49) 	&10.21 (0.68) 		& 11.52 (0.88)	&  8.39 (0.62)   		\\
        {\bf Protein}               & 6.12 (0.02) 	& 6.11 (0.02) 		& 6.12 (0.02)	&  5.06 (0.03)  		\\
        \hline
        \hline
    \end{tabular}
    
    \vspace{0.3cm}
    
    \begin{tabular}[width=\textwidth]{@{\extracolsep{1pt}}|r|r|r|r|r|@{}}
    \hline
    \hline
         {\bf NRF1 full} 	& {\bf NRF2 full}  		& {\bf NRF1 sparse} 	& {\bf NRF2 sparse}  & {\bf BART}		\\
        \hline

         \it{3.20(0.39)} 	& 3.35 (0.46)        		& 3.28 (0.41)	 	& 3.28 (0.42)  		& {\bf2.90 (0.33)}  	\\
         \it{4.34 (0.85)} 	& 4.68 (0.88)         		& 4.59 (0.91)	 	& 4.62 (0.88)  		& {\bf3.78 (0.51)}  	\\
         0.16 (0.01) 	& 0.16 (0.01)         		& \it{0.16 (0.01)}  	& 0.16 (0.01)    		& {\bf0.14 (0.01)} 	\\
     \bf{54.47 (34.64)} &78.60 (28.17) 			& 68.51 (35.56) 	& 82.80 (32.07) 	& 55.04 (16.40)  	\\
       	37.12 (2.89)	&41.22 (3.05) 			& 40.70 (2.51) 		& 38.03 (3.95)  		& {\bf 33.50 (2.26)}  	\\
       	 \it{6.28 (0.40)}	& 6.44 (0.37)      		&  7.42 (0.56)   		& 7.78 (0.56)       	& {\bf 5.20 (0.34)}  	\\
          4.82 (0.04)    	& {\it 4.77 (0.05)}      		&  4.82 (0.04)     	& 4.77 (0.05)   		& {\bf4.57 (0.04)}  	\\
        \hline
        \hline
    \end{tabular}
    \caption{RMSE test set results (and their standard deviation) for each of the models across the different data sets. ``Sparse'' stands for the sparsely connected NRF model, ``full'' for the fully connected. Best results are displayed in bold font, second best results in italic.}
    \label{results}
\end{table*}

        %

        
For most data sets, standard NN (regardless of the number of layers) do not achieve performance that is on par with the other models, though in some cases they give competitive RMSE score. Interestingly, NRF Method 1 seems to mostly outperform Method 2 with few exceptions. In fact, BART aside, the best overall performance is achieved by the fully connected NRF with Method 1, i.e., the averaging of individually trained per-tree networks. So one typically obtains bigger benefits from averaging several independently optimized single-tree networks than from jointly training one large network for all trees.
This means that ensemble averaging (Method 1) is more advantageous for the NRF model than the co-adaptation of parameters across trees (Method 2). We conjecture that by separating the optimization process for the individual networks, some implicit regularization occurs that reduces the vulnerability to overfitting when using Method 1. Additionally, it is important to note the excellent performance of BART, which is ranked first in 6 out of the 7 experiments. Although in this context BART should be seen as a benchmark, it remains that a good idea for future research is to use BART to design neural networks, and extend the approach presented in this article.

\paragraph{Training evolution.} To illustrate training behavior, the training and validation error for NRF (Method 2, fully connected) as well as for RF and NN from one of the experiments on the {\it Concrete Compressive Strength} data set are summarized in Figure \ref{fig:evolution}.
\begin{figure*}[!h]{{\extracolsep{8pt}}}
	\centering
	\includegraphics[width=0.45\textwidth]{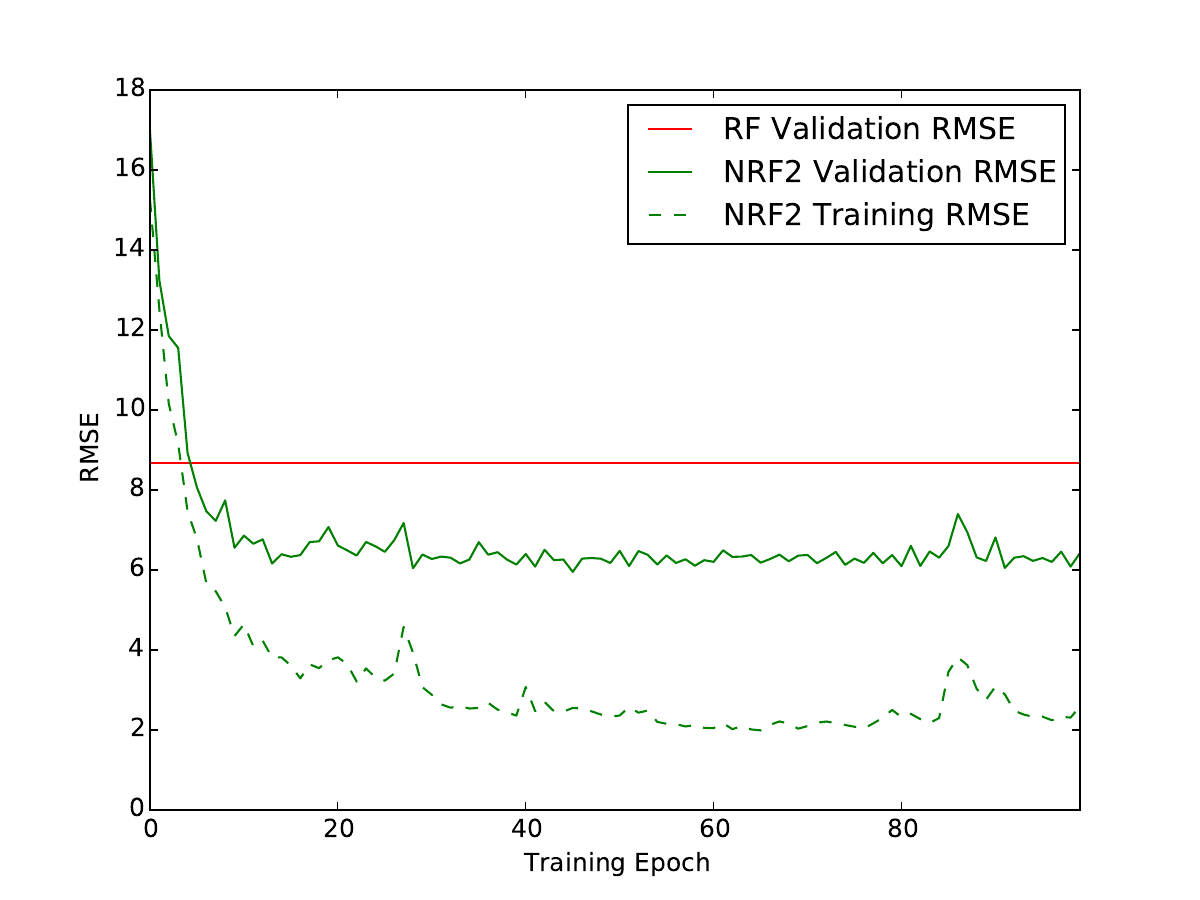}
	\includegraphics[width=0.45\textwidth]{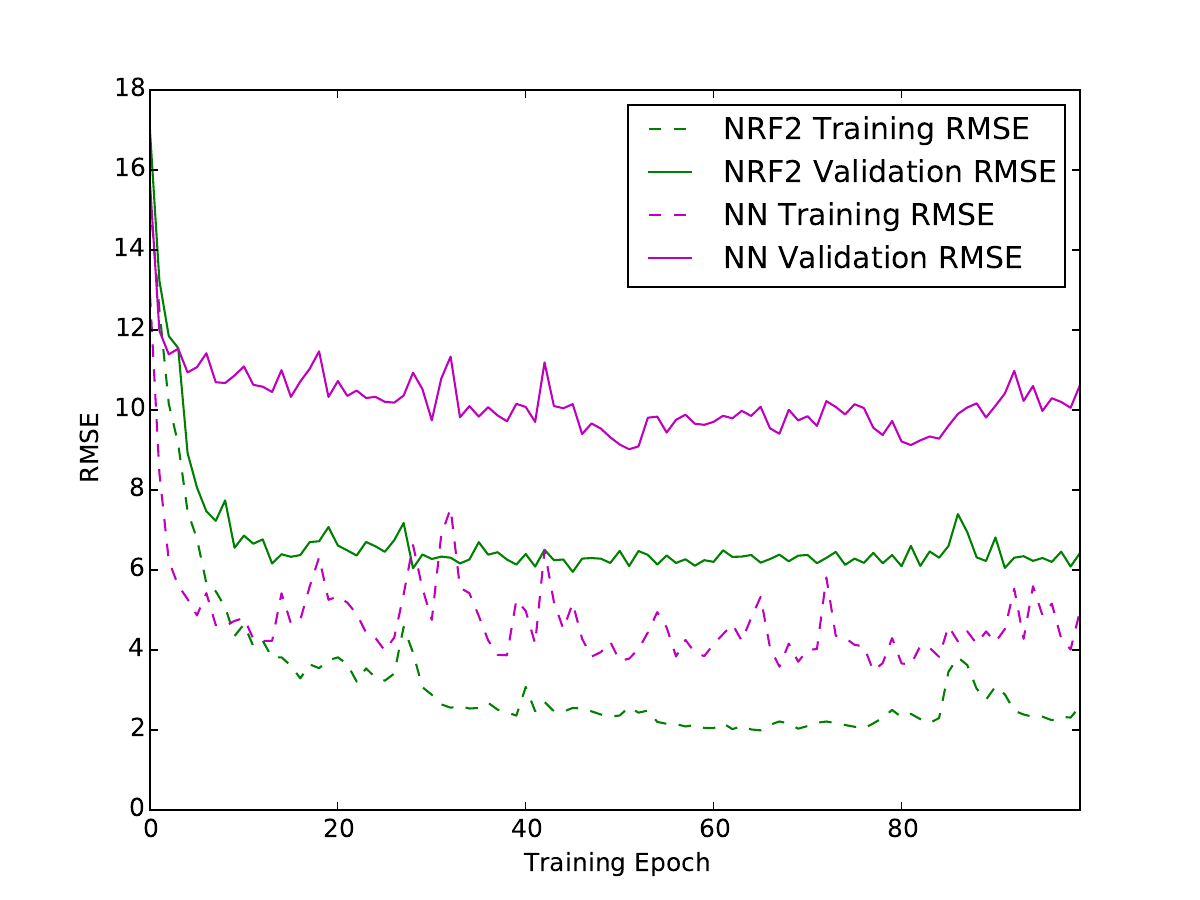}
	\caption{Evolution of training set and validation RMSE for different models. {\bf Left}: Validation RMSE for RF (red), along with validation and training set RMSE for NRF Method 2 (fully connected, green) across training epochs. {\bf Right}: Validation and training RMSE (solid and dotted lines, respectively) for NRF Method 2 (green) and a NN (purple) with same number of layers and neurons.}
	\label{fig:evolution}
\end{figure*}

Firstly, one can observe in the left plot that validation RMSE clearly drops below the level of RF RMSE during training. Note that the difference to the constant red line (RF validation error) represents the relative improvement over the RF on the validation set. So, in fact, network optimization helps the NRF to generate better predictions than the RF and this is achieved already very early in training. The initial RMSE value at epoch 0 is not exactly the same as for the RF, which is due to using smooth activation functions instead of sharp step functions (cf.~the contrast hyperparameters $\gamma_1$ and $\gamma_2$). 
In the plot on the right, the same NRF model behavior is shown again, but contrasted with a standard feed-forward NN of exactly the same size and number of parameters. Clearly, the NRF reaches much lower RMSE than the NN, both on training and validation set.


\subsection{Asymptotic behavior}
We also investigate the asymptotic behavior of the NRF model on an artificial data set created by sampling inputs $\bx$ uniformly from the $d$-dimensional hypercube $[0,1]^d$ and computing outputs $y$ as
\begin{equation*}
    y(x) = \sum_{j=1}^{d}\sin(20x^{(j)}-10) + \varepsilon,
\end{equation*}
where $\varepsilon$ is a zero mean Gaussian noise with variance $\sigma^2$, which corrupts the deterministic signal. We choose $d=2$ and $\sigma = 0.01$, and investigate the asymptotic behavior as the number of training samples increases. Figure \ref{fig:synthetic} illustrates the RMSE for an increasing number of training samples and shows that the NRF (Method 2, fully connected) error decreases much faster than the RF error as sample size increases.
\begin{figure}[h!]
\centering
\includegraphics[scale=0.6]{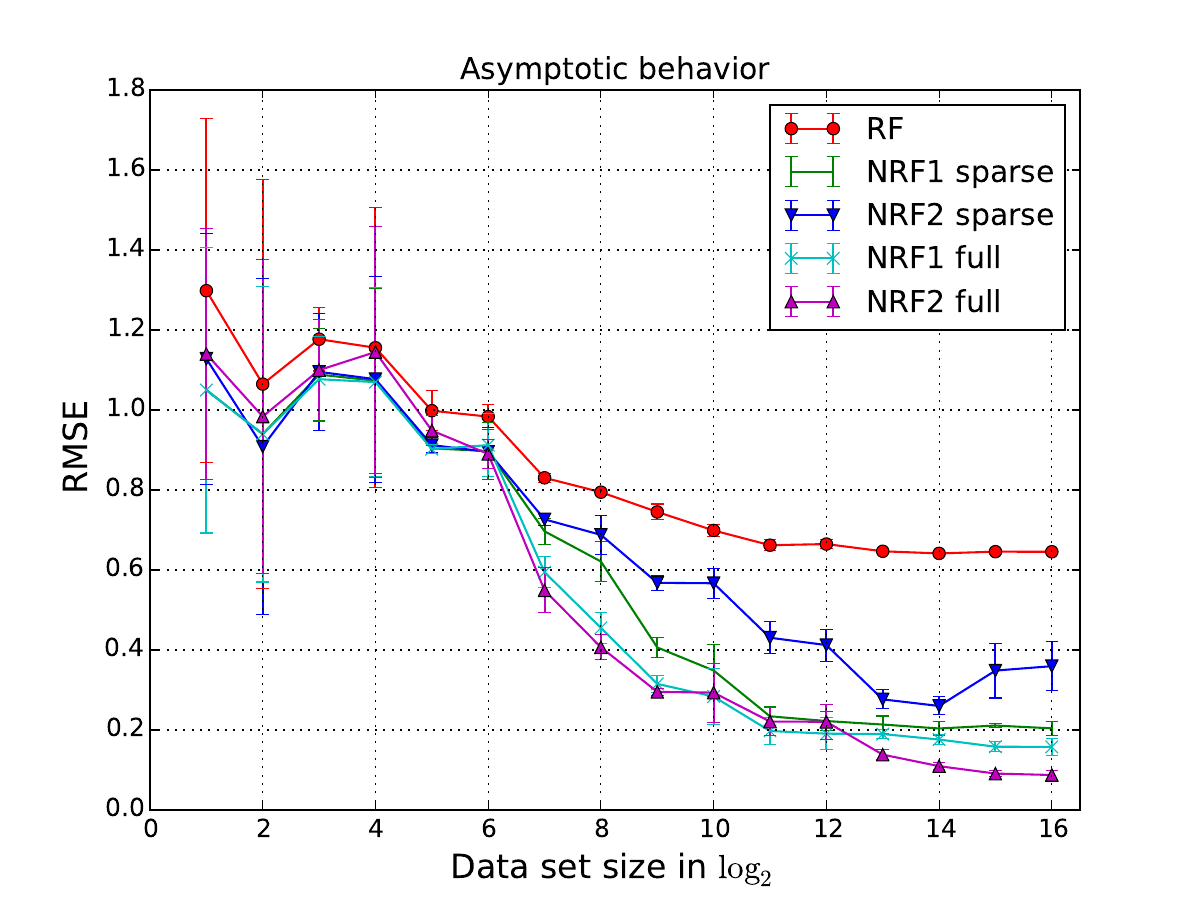}
\caption{This figure shows the test RMSE for synthetic data with exponentially increasing training set size ($x$-axis). Solid lines connect the mean RMSE values obtained across 3 randomly drawn datasets for each dataset size, whereas error bars show the empirical standard deviation; 30 trees, maximum depth 6, $\gamma_1=100$, $\gamma_2=1$.}
\label{fig:synthetic}
\end{figure}
\section{Some technical results}
\label{proofs}
\subsection{Proof of Proposition \ref{prop-ee}}
An easy adaptation of \citet[][Theorem 10.2]{GyKoKr02} shows that we can always assume that $\|Y\|_{\infty}<\infty$. Let $\mathbb M_{d,K_n-1}$ (resp., $\mathbb M_{K_n-1,K_n}$) be the vector space of $d\times (K_n-1)$ (resp., $(K_n-1)\times K_n$) matrices.
For a generic parameter
\begin{align*}
{\boldsymbol \lambda}=({\bf W}_{1}, {\bf b}_{1},& {\bf W}_{2}, {\bf b}_{2},{\bf W}_{\!\textrm{out}},{b}_{\textrm{out}})\\
&\in \mathbb M_{d,K_n-1} \times \mathds R^{K_n-1}\times \mathbb M_{K_n-1,K_n} \times \mathds R^{K_n}\times \mathds R^{K_n}\times \mathds R,
\end{align*}
we consider the constraint
\begin{equation}
\label{cpoids1bis}
\|{\bf W}_{2}\|_{\infty}+\|{\bf b}_{2}\|_{\infty}+\mathbf \|{\bf W}_{\!\textrm{out}}\|_1+|{b}_{\textrm{out}}| \leq C_1 K_n,
\end{equation}
and let
$$\Lambda=\big\{{\boldsymbol \lambda}=({\bf W}_{1},{\bf b}_{1},{\bf W}_{2},{\bf b}_{2},{\bf W}_{\!\textrm{out}},{b}_{\textrm{out}}): (\ref{cpoids1bis})\mbox{ is satisfied}\big\}.$$
We finally define
$$\mathscr F_n=\{f_{{\boldsymbol \lambda}}: {\boldsymbol \lambda} \in  \Lambda\},$$
where
$$f_{{\boldsymbol \lambda}}(\bx)={\bf W}_{\!\textrm{out}}^{\top}\sigma_2\Big({\bf W}_{2}^{\top} \sigma_1({\bf W}_{1}^{\!\top}\bx+{\bf b}_{1})+{\bf b}_{2}\Big)+{b}_{\textrm{out}},\quad \bx \in \mathbb R^d.$$
The set $\mathscr F_n$ contains all neural networks---constrained by (\ref{cpoids1bis})---with inputs in $\mathbb R^d$, two hidden layers of respective size $K_n-1$ and $K_n$, and one output unit. We note that $\mathscr F(\Theta,\mathscr D_n)\subseteq \mathscr F_n$ and that $\mathscr F_n$ is deterministic, in the sense that it does not depend neither on $\Theta$ nor on $\mathscr D_n$, but only on the size $n$ of the original data set. 

Clearly,
\begin{align*}
&\mathds E \sup_{f \in \mathscr F(\Theta,\mathscr D_n)}\bigg|\frac{1}{n}\sum_{i=1}^n\big|Y_i-f(\bX_i)\big|^2-\mathds E\big|Y-f(\bX)\big|^2\bigg| \\
& \quad \leq \mathds E \sup_{f \in \mathscr F_n}\bigg|\frac{1}{n}\sum_{i=1}^n\big|Y_i-f(\bX_i)\big|^2-\mathds E\big|Y-f(\bX)\big|^2\bigg|.
\end{align*}
According to (\ref{cpoids1bis}), each $f \in \mathscr F_n$ satisfies $\|f\|_{\infty}\leq C_1 K_n$. Thus, since $Y$ is assumed to be bounded, using uniformly bounded classes of functions we will be able to derive a useful exponential inequality.

Let $z_1^n=(z_1, \hdots, z_n)$ be a vector of $n$ fixed points in $\mathds R^d$ and let $\mathscr H$ be a set of functions from $\mathds R^d\to \mathds R$. For every $\varepsilon >0$, we let $\mathscr N_1(\varepsilon, \mathscr H,z_1^n)$ be the $L_1$ $\varepsilon$-covering number of $\mathscr H$ with respect to $z_1, \hdots, z_n$. Recall that $\mathscr N_1(\varepsilon, \mathscr H,z_1^n)$  is defined as the smallest integer $N$ such that there exist functions $h_1, \hdots, h_N:\mathds R^d\to \mathds R$ with the property that for every $h\in \mathscr H$ there is a $j=j(h)\in \{1, \hdots, N\}$ such that 
$$\frac{1}{n}\sum_{i=1}^n\big|h(z_i)-h_j(z_i)\big|<\varepsilon.$$
Note that if $Z_1^n = (Z_1,\hdots, Z_n)$ is a sequence of i.i.d.~random variables, then $\mathscr N_1(\varepsilon, \mathscr H,Z_1^n)$ is a random variable as well. 

Now, let $Z=(\bX,Y)$, $Z_1=(\bX_1,Y_1), \hdots, Z_n=(\bX_n,Y_n)$, and
\begin{align*}
\mathscr H_n=\big\{h(\bx,y)& :=|y-f(\bx)\big|^2: \\
& (\bx,y)\in [0,1]^d\times [-\|Y\|_{\infty},\|Y\|_{\infty}]\mbox{ and }f\in \mathscr F_n\big\}.
\end{align*}
Note that the functions in $\mathscr H_n$ satisfy 
$$0\leq h(\bx,y)\leq 2C_1^2K_n^2+2\|Y\|_{\infty}^2.$$
In particular,
$$0\leq h(\bx,y)\leq 4C_1^2K_n^2$$
if $n$ is large enough such that $C_1K_n\geq \|Y\|_{\infty}$ is satisfied, which we assume.

According to an inequality of \citet{Po84} \citep[see also][Theorem 9.1]{GyKoKr02}, we have, for arbitrary $\varepsilon >0$,
\begin{align}
&\mathds P\bigg\{ \sup_{f \in \mathscr F_n}\Big|\frac{1}{n}\sum_{i=1}^n\big|Y_i-f(\bX_i)\big|^2-\mathds E\big|Y-f(\bX)\big|^2\Big|>\varepsilon\bigg\}\nonumber\\
& \quad =\mathds P\bigg\{ \sup_{h \in \mathscr H_n}\Big|\frac{1}{n}\sum_{i=1}^nh(Z_i)-\mathds Eh(Z)\Big|>\varepsilon\bigg\}\nonumber\\
& \quad \leq 8 \mathbb E \mathscr N_1\big(\frac{\varepsilon}{8}, \mathscr H_n,Z_1^n\big)e^{-\frac{n\varepsilon^2}{2048C_1^4K_n^4}}.\label{zero}
\end{align}
So, we have to upper bound $\mathds E \mathscr N_1(\frac{\varepsilon}{8}, \mathscr H_n, Z_1^n\big)$. To begin with, consider two functions $h(\bx,y)= |y-f(\bx)|^2$ and $h_1(\bx,y)= |y-f_1(\bx)|^2$ of $\mathscr H_n$. Clearly,
\begin{align*}
&\frac{1}{n}\sum_{i=1}^n\big|h(Z_i)-h_1(Z_i)\big|\\
&\quad =\frac{1}{n}\sum_{i=1}^n\Big|\big|Y_i-f(\bX_i)\big|^2-\big|Y_i-f_1(\bX_i)\big|^2\Big|\\
& \quad =\frac{1}{n}\sum_{i=1}^n\big|f(\bX_i)-f_1(\bX_i)\big|\times \big|f(\bX_i)+f_1(\bX_i)-2Y_i\big|\\
& \quad  \leq \frac{4C_1K_n}{n}\sum_{i=1}^n\big|f(\bX_i)-f_1(\bX_i)\big|.
\end{align*}
Thus,
\begin{equation}
\label{one}
\mathscr N_1\Big(\frac{\varepsilon}{8}, \mathscr H_n, Z_1^n\Big)\leq \mathscr N_1\Big(\frac{\varepsilon}{64C_1K_n}, \mathscr F_n, \bX_1^n\Big).
\end{equation}
The covering number $\mathscr N_1(\frac{\varepsilon}{64C_1K_n}, \mathscr F_n, \bX_1^n)$ can be upper bounded independently of $\bX_1^n$ by extending the arguments of \citet[][]{LuZe95} from a network with one hidden layer to a network with two hidden layers. In the arguments below, we repeatedly apply \citet[][Theorem 9.4, Lemma 16.4, and Lemma 16.5]{GyKoKr02}. The neurons of the first hidden layer output functions that belong to the class
$$\mathscr G_1=\{\sigma_1(a^{\top} \bx+a_0): a\in \mathds R^d, a_0\in \mathds R\},$$
and it is easy to show that, for $\varepsilon \in (0,1/4)$,
$$\mathscr N_1(\varepsilon, \mathscr G_1,\bX_1^n) \leq 9\Big(\frac{6e}{\varepsilon}\Big)^{4d+8}.$$
Next, letting 
$$\mathscr G_2=\big\{bg: g \in \mathscr G_1, |b|\leq C_1K_n\big\},$$
we get
\begin{align*}
\mathscr N_1(\varepsilon, \mathscr G_2,\bX_1^n) &\leq \frac{4C_1K_n}{\varepsilon}\mathscr N_1\Big(\frac{\varepsilon}{2C_1K_n}, \mathscr G_1,\bX_1^n\Big)\\
& \leq \Big(\frac{36 e C_1 K_n}{\varepsilon}\Big)^{4d+9}.
\end{align*}
The  second units compute functions of the collection
$$\mathscr G_3=\Big\{\sigma_2\Big(\sum_{i=1}^{K_n-1}g_i+b_0\Big): g_i \in \mathscr G_2, |b_0|\leq C_1K_n\Big\}.$$
Note that $\sigma_2$ satisfies the Lipschitz property $|\sigma_2(u)-\sigma_2(v)| \leq \gamma_2|u-v|$ for all $(u,v)\in \mathds R^2$. Thus, 
\begin{align*}
\mathscr N_1(\varepsilon, \mathscr G_3,\bX_1^n) &\leq \frac{2C_1\gamma_2K_n^2}{\varepsilon}\mathscr N_1\Big(\frac{\varepsilon}{2\gamma_2K_n}, \mathscr G_2,\bX_1^n\Big)^{K_n-1}\\
& \leq \Big( \frac{72eC_1\gamma_2K_n^2}{\varepsilon}\Big)^{(4d+9)K_n+1}.
\end{align*}
Also, letting
$$\mathscr G_4=\{wg: g\in \mathscr G_3,|w|\leq C_1K_n\},$$
we see, assuming without loss of generality $C_1,\gamma_2 \geq 1$, that
\begin{align*}
\mathscr N_1(\varepsilon, \mathscr G_4,\bX_1^n) &\leq \frac{4C_1K_n}{\varepsilon}\mathscr N_1\Big(\frac{\varepsilon}{2C_1K_n}, \mathscr G_3,\bX_1^n\Big)\\
& \leq \Big( \frac{144eC_1^2\gamma_2K_n^3}{\varepsilon}\Big)^{(4d+9)K_n+2}.
\end{align*}
Finally, upon noting that
$$\mathscr F_n=\Big\{\sum_{i=1}^{K_n}g_i+{b}_{\textrm{out}} :  g_i \in \mathscr G_4, |{b}_{\textrm{out}}|\leq C_1K_n\Big\},$$
we conclude
\begin{align}
\mathscr N_1(\varepsilon, \mathscr F_n,\bX_1^n)  & \leq \frac{2C_1K_n(K_n+1)}{\varepsilon}\mathscr N_1\Big(\frac{\varepsilon}{K_n+1}, \mathscr G_4,\bX_1^n\Big)^{K_n}\nonumber\\
& \leq \Big( \frac{144eC_1^2\gamma_2(K_n+1)^4}{\varepsilon}\Big)^{(4d+9)K_n^2+2K_n+1}.\label{two}
\end{align}
Combining inequalities (\ref{zero})-(\ref{two}), we obtain
\begin{align*}
&\mathds P\bigg\{ \sup_{f \in \mathscr F_n}\Big|\frac{1}{n}\sum_{i=1}^n\big|Y_i-f(\bX_i)\big|^2-\mathds E\big|Y-f(\bX)\big|^2\Big|>\varepsilon\bigg\} \\
& \quad \leq 8\Big( \frac{9216eC_1^3\gamma_2(K_n+1)^5}{\varepsilon}\Big)^{(4d+9)K_n^2+2K_n+1}e^{-\frac{n\varepsilon^2}{2048C_1^4K_n^4}}.
\end{align*}
Therefore, for any $\varepsilon \in (0,1/4)$, and all $n$ large enough, 
\begin{align*}
&\mathds E\sup_{f \in \mathscr F_n}\Big|\frac{1}{n}\sum_{i=1}^n\big|Y_i-f(\bX_i)\big|^2-\mathds E\big|Y-f(\bX)\big|^2\Big|\\
& \leq \varepsilon + 8 \int_{\varepsilon}^{\infty} \Big( \frac{9216eC_1^3\gamma_2(K_n+1)^5}{t}\Big)^{(4d+9)K_n^2+2K_n+1}e^{-\frac{nt^2}{2048C_1^4K_n^4}}\,\mbox{d}t\\
& \leq \varepsilon+ 8  \Big( \frac{9216eC_1^3\gamma_2(K_n+1)^5}{\varepsilon}\Big)^{(4d+9)K_n^2+2K_n+1}\\
& \hspace{5cm}\times \Big[ - \frac{2048C_1^4K_n^4}{n\varepsilon}e^{-\frac{n\varepsilon t}{2048C_1^4K_n^4}}\Big]_{t=\varepsilon}^{\infty}\\
&   \leq \varepsilon+ 8  \Big( \frac{9216eC_1^3\gamma_2(K_n+1)^5}{\varepsilon}\Big)^{(4d+9)K_n^2+2K_n+1}\\
& \hspace{5cm}\times\Big(\frac{2048C_1^4K_n^4}{n\varepsilon}\Big)e^{-\frac{n\varepsilon ^2}{2048C_1^4K_n^4}}\\
&  \leq \varepsilon + 8\Big(\frac{2048C_1^4K_n^4}{n\varepsilon}\Big)\\
& \hspace{1cm} \times \exp \Big[\big((4d+9)K_n^2+2K_n+1\big)\log \Big(\frac{9216eC_1^3\gamma_2(K_n+1)^5}{\varepsilon}\Big)\\
& \hspace{2.5cm} -\frac{n\varepsilon ^2}{2048C_1^4K_n^4} \Big].
\end{align*}
Thus, under the conditions $K_n, \gamma_2 \to \infty$, and \
$K_n^6\log (\gamma_2 K_n^5)/n\to 0$, for all $n$ large enough, 
$$\mathds E\sup_{f \in \mathscr F_n}\Big|\frac{1}{n}\sum_{i=1}^n\big|Y_i-f(\bX_i)\big|^2-\mathds E\big|Y-f(\bX)\big|^2\Big|\leq 2 \varepsilon$$
and so,
$$\mathds E \sup_{f \in \mathscr F(\Theta,\mathscr D_n)}\Big|\frac{1}{n}\sum_{i=1}^n\big|Y_i-f(\bX_i)\big|^2-\mathds E\big|Y-f(\bX)\big|^2\Big|\leq 2 \varepsilon.$$
Since $\varepsilon$ was arbitrary, the proof is complete.
\subsection{Proof of Proposition \ref{prop-lambda}}
By definition,
$$t_{\boldsymbol \lambda^{\star}}(\bx)={\bf W}^{\star \top}_{\!\textrm{out}}\tau\Big({\bf W}_{2}^{\star\top} \tau({\bf W}_{1}^{\star\top}\bx+{\bf b}_{1}^{\star})+{\bf b}_{2}^{\star}\Big)+b_{\textrm{out}}^{\star}$$
and 
$$f_{\boldsymbol \lambda^{\star}}(\bx)={\bf W}^{\star \top}_{\!\textrm{out}}\sigma_2\Big({\bf W}_{2}^{\star\top} \sigma_1({\bf W}_{1}^{\star\top}\bx+{\bf b}_{1}^{\star})+{\bf b}_{2}^{\star}\Big)+b_{\textrm{out}}^{\star}.$$
We have, for all $\bx \in \mathbb R^d$,
\begin{align*}
&\big|f_{\boldsymbol \lambda^{\star}}(\bx)-t_{\boldsymbol \lambda^{\star}}(\bx)\big|^2 \\
& \quad \leq \|{\bf W}_{\!\textrm{out}}^{\star}\|^2\times \Big\|\sigma_2\Big({\bf W}_{2}^{\star\top} \sigma_1({\bf W}_{1}^{\star\top}\bx+{\bf b}_{1}^{\star})+{\bf b}_{2}^{\star}\Big)\\
& \hspace{3.3cm} -\tau\Big({\bf W}_{2}^{\star\top} \tau({\bf W}_{1}^{\star\top}\bx+{\bf b}_{1}^{\star})+{\bf b}_{2}^{\star}\Big)\Big\|^2\\
& \qquad \mbox{(by the Cauchy-Schwarz inequality})\\
& \quad \leq \frac{\|Y\|_{\infty}^2 K_n }{4}\times\Big\|\sigma_2\Big({\bf W}_{2}^{\star\top} \sigma_1({\bf W}_{1}^{\star\top}\bx+{\bf b}_{1}^{\star})+{\bf b}_{2}^{\star}\Big)\\
& \hspace{4cm} -\tau\Big({\bf W}_{2}^{\star\top} \tau({\bf W}_{1}^{\star\top}\bx+{\bf b}_{1}^{\star})+{\bf b}_{2}^{\star}\Big)\Big\|^2.
\end{align*}
By the triangle inequality,
\begin{align*}
&\Big\|\sigma_2\Big({\bf W}_{2}^{\star\top} \sigma_1({\bf W}_{1}^{\star\top}\bx+{\bf b}_{1}^{\star})+{\bf b}_{2}^{\star}\Big)-\tau\Big({\bf W}_{2}^{\star\top} \tau({\bf W}_{1}^{\star\top}\bx+{\bf b}_{1}^{\star})+{\bf b}_{2}^{\star}\Big)\Big\|^2\\
& \leq 2\Big\|\sigma_2\Big({\bf W}_{2}^{\star\top} \tau({\bf W}_{1}^{\star\top}\bx+{\bf b}_{1}^{\star})+{\bf b}_{2}^{\star}\Big)\\
& \hspace{1.2cm}-\tau\Big({\bf W}_{2}^{\star\top} \tau({\bf W}_{1}^{\star\top}\bx+{\bf b}_{1}^{\star})+{\bf b}_{2}^{\star}\Big)\Big\|^2\\
&\quad+ 2\Big\|\sigma_2\Big({\bf W}_{2}^{\star\top} \sigma_1({\bf W}_{1}^{\star\top}\bx+{\bf b}_{1}^{\star})+{\bf b}_{2}^{\star}\Big)\\
&  \hspace{1.5cm}-\sigma_2\Big({\bf W}_{2}^{\star\top} \tau({\bf W}_{1}^{\star\top}\bx+{\bf b}_{1}^{\star})+{\bf b}_{2}^{\star}\Big)\Big\|^2\\
& \quad :={\bf I} + {\bf II}.
\end{align*}
Upon noting that, for all $u\in \mathds R$, $|\sigma_2(u)-\tau(u)|\leq 2e^{-2\gamma_2 |u|}$, we see that
$${\bf I}\leq8\sum_{i=1}^{K_n}  \exp \Big[-4\gamma_2\big|\big({\bf W}_{2}^{\star\top} \tau({\bf W}_{1}^{\star\top}\bx+{\bf b}_{1}^{\star})+{\bf b}_{2}^{\star}\big)_i\big|\Big].$$
But, by the very definition of $({\bf W}_{1}^{\star},{\bf b}_{1}^{\star},{\bf W}_{2}^{\star},{\bf b}_{2}^{\star})$---see (\ref{def_threshold_function})---we have, for every $i$, 
$$\big|\big({\bf W}_{2}^{\star\top} \tau({\bf W}_{1}^{\star\top}\bx+{\bf b}_{1}^{\star})+{\bf b}_{2}^{\star}\big)_i\big|\geq 1/2.$$
Thus, ${\bf I}\leq 8K_ne^{-2\gamma_2}$.

For the term ${\bf II}$, note that $\sigma_2$ satisfies the Lipschitz property $|\sigma_2(u)-\sigma_2(v)| \leq \gamma_2|u-v|$ for all $(u,v)\in \mathds R^2$. Hence,
\begin{align*}
{\bf II} & \leq 2\sum_{i=1}^{K_n}\gamma_2^2\Big|\Big({\bf W}_{2}^{\star\top} \big ( \sigma_1({\bf W}_{1}^{\star\top}\bx+{\bf b}_{1}^{\star})-\tau({\bf W}_{1}^{\star\top}\bx+{\bf b}_{1}^{\star})\Big)_i\Big|^2\\
& \leq 2\gamma_2^2K_n^2 \big\|\sigma_1({\bf W}_{1}^{\star\top}\bx+{\bf b}_{1}^{\star})-\tau({\bf W}_{1}^{\star\top}\bx+{\bf b}_{1}^{\star})\big\|^2\\
& \quad \mbox{(by the Cauchy-Schwarz inequality and the definition of ${\bf W}_{2}^{\star}$)}\\
& \leq 8\gamma_2^2K_n^2 \sum_{i=1}^{K_n-1}e^{-4\gamma_1|({\bf W}_{1}^{\star\top}\bx+{\bf b}_{1}^{\star})_i|}.
\end{align*}
For fixed $i$ and arbitrary $\varepsilon >0$, we have
\begin{align*}
&\int_{[0,1]^d}e^{-4\gamma_1|({\bf W}_{1}^{\star\top}\bx+{\bf b}_{1}^{\star})_i|}\mu(\mbox{d}\bx)\\
&\quad =\int_{[0,1]}e^{-4\gamma_1 |x^{(j_n^{\star})}-\alpha_n^{\star}|}\mu(\mbox{d}\bx)\\
& \qquad \mbox{(for some $j \in \{1, \hdots, d\}$, depending upon $\Theta$ and $\mathscr D_n$)}\\
&\quad \leq e^{-4\gamma_1\varepsilon}+2\varepsilon,
\end{align*}
since $\bX$ is uniformly distributed in $[0,1]^d$. We see that, for all $n$ large enough, choosing $\varepsilon =\frac{\log(2\gamma_1)}{4\gamma_1}$,
$$\int_{[0,1]^d}e^{-4\gamma_1|({\bf W}_{1}^{\star\top}\bx+{\bf b}_{1}^{\star})_i|}\mu(\mbox{d}\bx)\leq \frac{\log(2\gamma_1)}{\gamma_1}.$$
Putting all the pieces together, we conclude that
\begin{align*}
\mathbb E\int_{[0,1]^d}\big|f_{{\boldsymbol \lambda}^{\star}}(\bx)-t_{{\boldsymbol \lambda}^{\star}}(\bx)\big|^2 \mu(\mbox{d}\bx) & \leq 2\|Y\|_{\infty}^2 K_n^2 \Big(e^{-2\gamma_2}+ \frac{\gamma_2^2K_n^2\log(2\gamma_1)}{\gamma_1}\Big).
\end{align*}
The upper bound tends to zero under the conditions of the proposition.
\subsection{Proof of Proposition \ref{prop-diameter}}
The proof of Proposition \ref{prop-diameter} rests upon the following lemma, which is a multivariate extension of Technical Lemma $1$ in \citet{ScBiVe15}.
\begin{lem} \label{proof_tech_lemma_1}
Assume that $\bX$ is uniformly distributed in $[0,1]^d$, $\|Y\|_{\infty}<\infty$, and $r \in \mathscr{F}$. Assume, in addition, that $L^{\star} \equiv 0$ for all cuts in some given nonempty cell $A$.
Then the regression function $r$ is constant on $A$.
\end{lem}
\begin{proof} [Proof of Lemma \ref{proof_tech_lemma_1}]
We start by proving the result in dimension $d=1$. Letting $A = [a,b]$ $(0\leq a< b\leq 1$), one has
\begin{align*}
L^{\star}(1,z) & = \V[ Y | \bX \in A] - \P[ a \leq \bX \leq z \,|\, \bX\in A] \V[Y \,|\, a \leq \bX \leq z ] \\
& \quad -  \P[ z \leq \bX \leq b \,|\, \bX\in A] \V[Y \, | \, z < \bX \leq b ]\\
& = - \frac{1}{(b-a)^2} \bigg( \int_a^b r(t) \textrm{d}t\bigg)^2 + \frac{1}{(b-a)(z-a)}\bigg( \int_a^z r(t) \textrm{d}t\bigg)^2 \\
& \quad + \frac{1}{(b-a)(b-z)} \bigg(  \int_z^b r(t) \textrm{d}t\bigg)^2.
\end{align*}
Let $C = \int_a^b r(t)\textrm{d}t$ and $R(z) = \int_a^z r(t)\textrm{d}t$. Simple calculations show that
\begin{align*}
L^{\star}(1,z) & = \frac{1}{(z-a)(b-z)} \bigg(  R(z) - C\,\frac{z-a}{b-a} \bigg)^2.
\end{align*}
Thus, since $ L^{\star} \equiv 0$ on $\mathscr{C}_A$ by assumption, we obtain
\[
R(z) = C\,\frac{z-a}{b-a}.
\]
This proves that $R(z)$ is linear in $z$, and therefore that $r$ is constant on $[a,b]$. 

Let us now examine the general multivariate case, where the cell is a hyperrectangle $A= \Pi_{i = 1}^d [a_i,b_i] \subseteq [0,1]^d $. We recall the notation $A^{\backslash j} = \prod_{i \neq j} [a_i, b_i]$ and $\textrm{d} \bx^{\backslash j} = \textrm{d}x_1 \hdots \textrm{d}x_{j-1} \textrm{d}x_{j+1} \hdots \textrm{d}x_d$. From the univariate analysis above, we know that for all $j \in \{1, \hdots, d\}$ there exists a constant $C_j$ such that 
\begin{align*}
\int_{A^{\backslash j}} r(\bx) \textrm{d} \bx^{\backslash j} = C_j.
\end{align*}
Since $r \in \mathscr{F}$, it is constant on $A$. This proves the result. \end{proof}

The proof of the next proposition follows the arguments of Proposition $2$ in \citet{ScBiVe15}, with Lemma \ref{proof_tech_lemma_1} above in lieu of their Technical Lemma $1$. The proof is therefore omitted. We let $A_n(\bX, \Theta)$ be the cell of the tree grown with random parameter $\Theta$ that contains $\bX$.
\begin{pro} \label{prop_equiv_prop2}
Assume that $\bX$ is uniformly distributed in $[0,1]^d$, $\|Y\|_{\infty}<\infty$, and $r \in \mathscr{F}$. Then, for all $\rho, \xi>0$, there exists $N \in \mathds{N}^{\star}$ such that, for all $n >N$,
\begin{align*}
 \mathds{P} \bigg[ \sup\limits_{\bz, \bz' \in A_{n}(\bX, \Theta)}\big|r(\bz) - r(\bz')\big| \leq \xi \bigg] \geq 1 - \rho.
\end{align*}
\end{pro}

We are now in a position to prove Proposition \ref{prop-diameter}. Recall that 
$$
t_{\boldsymbol \lambda^{\star}}(\bx) = \mathds{E}[Y | \bX \in A_n(\bx, \Theta)].
$$
Accordingly, 
\begin{align*}
\big|t_{\boldsymbol \lambda^{\star}}(\bx)-r(\bx)\big|^2
& = \bigg|\mathds{E}\big[Y \,|\, \bX \in A_n(\bx, \Theta)\big] -r(\bx)\bigg|^2 \\
& = \bigg|\mathds{E}\big[r(\bX) \,|\, \bX \in A_n(\bx, \Theta)\big] -r(\bx)\bigg|^2 \\
& \leq  \sup\limits_{\bz, \bz' \in A_n(\bx, \Theta)}\big|r(\bz) - r(\bz')\big|^2,
\end{align*}
since $r$ is continuous on $[0,1]^d$. Taking the expectation on both sides, we finally obtain
\begin{align*}
\mathds{E} \big|t_{\boldsymbol \lambda^{\star}}(\bX)-r(\bX)\big|^2 \leq \mathds{E} \bigg[ \sup\limits_{\bz, \bz' \in A_n(\bX, \Theta)}\big|r(\bz) - r(\bz')\big|^2 \bigg],
\end{align*}
which tends to zero as $n \to \infty$, according to Proposition \ref{prop_equiv_prop2} (since $r$ is bounded on $[0,1]^d$).

\paragraph{Acknowledgments.} We greatly thank the Associate Editor and the Referee for valuable comments and insightful suggestions, which lead to a substantial improvement of the paper.

\bibliography{biblio-bsw}

\end{document}